\documentclass{article}
\usepackage{ijcai21}

\usepackage{times}

\usepackage{soul}
\usepackage{url}
\usepackage[hidelinks]{hyperref}
\usepackage[utf8]{inputenc}
\usepackage[small]{caption}
\usepackage{graphicx}

\usepackage{booktabs}
\usepackage{amsfonts}       % blackboard math symbols
\usepackage{amsmath}
\usepackage{amsthm}
\usepackage{graphicx}
\usepackage{bm}
\usepackage{xcolor}

\usepackage{subcaption}
\graphicspath{ {./images/} }
\usepackage{tkz-graph}

\newtheorem{theorem}{Theorem}
\newtheorem{lemma}{Lemma}
\newtheorem{corollary}{Corollary}
\theoremstyle{definition}
\newtheorem{definition}{Definition}
\newtheorem{remark}{Remark}
\newtheorem{example}{Example}

\newcommand{\Y}{\mathcal{Y}}
\newcommand{\X}{\mathcal{X}}
\newcommand{\Z}{\mathcal{Z}}
\newcommand{\F}{\mathcal{H}}
\newcommand{\R}{\mathbb{R}}
\newcommand{\N}{\mathcal{N}_\infty}
\newcommand{\D}{\mathcal{D}}
\newcommand{\Pro}{\mathbb{P}}
\newcommand{\E}{\mathbb{E}}
\newcommand{\Ne}{\mathcal{N}}

\newcommand{\Rad}{\mathfrak{R}}
\newcommand{\I}{\mathbb{I}}

\newcommand{\dotp}[2]{\left< #1, #2\right>}

\title{Fine-grained Generalization Analysis of Structured Output Prediction\footnote{To appear in IJCAI 2021} }

\author{
Waleed Mustafa$^1$\and
Yunwen Lei$^2$\and
Antoine Ledent$^{1}$\And
Marius Kloft$^1$\\
\affiliations
$^1$TU Kaiserslautern\\
$^2$ University of Birmingham\\
\emails
mustafa@cs.uni-kl.de,
y.lei@bham.ac.uk,
\{ledent, kloft\}@cs.uni-kl.de
}

\begin{document}

\maketitle

\begin{abstract}
In machine learning we often encounter structured output prediction problems (SOPPs), i.e. problems where the output space admits a rich internal structure. Application domains where SOPPs naturally occur include natural language processing, speech recognition, and computer vision. Typical SOPPs have an extremely large label set, which grows exponentially as a function of the size of the output. Existing generalization analysis implies generalization bounds with at least a square-root dependency on the cardinality $d$ of the label set, which can be vacuous in practice. In this paper, we significantly improve the state of the art by developing novel high-probability bounds with a logarithmic dependency on $d$. Moreover, we leverage the lens of algorithmic stability to develop generalization bounds in expectation without any dependency on $d$. Our results therefore build a solid theoretical foundation for learning in large-scale SOPPs. Furthermore, we extend our results to learning with weakly dependent data. % and show that the logarithmic dependency still holds in this setting.

\end{abstract}

\section{Introduction}
%{\color{red}Would you please draft the introduction. In the first paragraph, we can say something about structured prediction, eg its importance. In the second paragraph, we can say some drawback of the existing study in structured prediction. %Also we require a section on related work.}
Structured output prediction (SOP) refers to a broad class of machine learning problems with a rich structure in the output space. For instance, the output may be a sequence of tags in part-of-speech (POS) tagging, a sentence in machine translation, or a grid of segmentation labels in image segmentation.

A distinguishing property of these tasks is that the loss function  admits a decomposition along the output structures. For instance, if the output is a sequence of partial labels, the loss function could be the Hamming distance. The output structure makes those problems substantially different, both algorithmically and theoretically, from well-studied machine-learning methods such as binary classification. Algorithms specifically targeted at SOPPs have been put forward in~\cite{lafferty2001conditional,ciliberto2016consistent,taskar2003max,tsochantaridis2005large,alg4,alg5,chen2017deeplab}, to mention but a few. 

Whilst the subject of SOP is well explored from a practical point of view, existing theoretical analyses have several limitations. For instance, the results in~\cite{taskar2003max,collins2001parameter} apply only to specific factor graphs and bound errors measured only by the Hamming loss, while other losses such as \textit{edit distance} and \textit{BLUE scores} are more natural in many applications. \cite{mcallester2007generalization} introduced guarantees that apply to general losses but only to randomized linear algorithms and admit only a square-root dependence on the size of substructures. In \cite{cortes2016structured}, the authors introduced general bounds that apply to general factor graphs and general losses from the viewpoint of function class capacity. However, the associated bounds exhibit a square-root dependence on the number $d$ of categories a subset of substructures can take, which can become vacuous when applied to extreme multi-class contexts \cite{lei2019data} or models that assume a large dependence between the substructures. 

In this paper, we aim to advance the state of the art in the theoretical foundation of SOP by developing generalization bounds applicable to large-scale problems with millions of labels.
Our contributions are as follows.

\noindent 1. We apply the celebrated technique of Rademacher complexity to develop high-probability generalization bounds with a \emph{log} dependency on the size of the label set. This substantially improves the existing state of the art, which comes with at least a square-root dependency. We achieve this improvement by using covering numbers measured by the $\ell_{\infty}$-norm, which can exploit the Lipschitz continuity of loss functions with respect to (w.r.t.) the  $\ell_{\infty}$-norm. For comparison, the existing complexity analysis uses the Lipschitz continuity w.r.t. the $\ell_2$-norm~\cite{cortes2016structured}, which does not match the regularity of loss functions in structured output prediction and thus leads to suboptimal bounds.

\noindent 2. We leverage the framework of algorithmic stability to further remove the log dependency for generalization bounds in expectation. We consider two popular methods for structured output prediction: stochastic gradient descent (SGD) and regularized risk minimization (RRM). We adapt the existing stability analysis in a way to exploit the Lipschitz continuity w.r.t. the $\ell_\infty$-norm of loss functions in SOP.

\noindent 3. We extend our discussion to learning with weakly dependent training examples, which are widespread in SOPPs. For example, in natural language processing (NLP), a data set can come in the form of sets of documents, while learning is performed at the sentence level. While assuming that the sentences are independent is inaccurate, it is reasonable to assume that the dependency between sentences decreases when their distance in a document increases.

%{\color{red} Would you please add some discussion here}.

The remaining parts of the paper are structured as follows. We discuss some related work in Section \ref{sec:work} and present the problem formulation in Section \ref{sec:formulation}. We present our main results on generalization bounds in Section \ref{sec:result}, which are extended to learning with dependent examples in Section \ref{sec:dependent}. We conclude the paper in Section \ref{sec:conclusion}.

\section{Related Work\label{sec:work}}
We first review some work on structured output prediction. Many algorithms have been developed to solve structured output prediction problems. Early techniques considered generative probabilistic models (e.g.,  hidden Markov models \cite{rabiner1986introduction}). Motivated by the success of support vector machines (SVM), large-margin models for structured data were proposed in \cite{taskar2003max,tsochantaridis2005large}. To reduce the model complexity, conditional random fields (CRFs) \cite{lafferty2001conditional} model the conditional distribution of the structured outputs rather than modeling the joint probability of the input and output. A key property of these models is that their prediction step can be viewed as maximising a scoring function. Such a scoring function enjoys a decomposition over the substructure so that the maximisation can be done efficiently. CRFs were combined with convolutional neural networks (CNNs) in \cite{chen2017deeplab} to approach semantic segmentation problems, achieving better performance than CNNs alone. %given the input

In \cite{collins2001parameter,taskar2003max}, the authors showed a generalization bound for their proposed models. However, they considered restricted models and losses (Hamming loss).
%They, however, apply to restricted models and are only for the hamming loss. 
A PAC-Bayesian bound is proved in \cite{mcallester2007generalization} for Bayesian prediction algorithms. In \cite{cortes2016structured} the authors introduced a more general generalization bound that applies to general losses and models. Their bound scales as the square root of the number of classes. This can lead to vacuous bounds when the number of classes per substructure and their dependence on each other continue to increase. % see Figure~\ref{fig:factor_graphs}.    

\cite{ciliberto2016consistent} introduced the \textit{implicit embedding} approach to structured output prediction where the label is encoded into a vector in some Hilbert space via an encoding function. A decoding function is also defined so that prediction is performed by composing a regression function and the decoding function, thus establishing a connection between structured output prediction and regression. They provided generalization bounds of the order of $O(m^{-\frac{1}{4}})$, where $m$ is the number of samples, which can be a problem for large $m$. Recently, \cite{ciliberto2019localized} introduced the setting of localized structured output prediction, where they assume a form of weak dependence between substructures. Their model utilizes such assumption by treating each part of the structure as an independent sample. They prove bounds of the order $O((ml)^{-\frac{1}{4}})$ for their method, where $l$ is the number of substructures, under weakly dependent samples.

We now review the related work for multi-class classification (MCC), which is a specific case of structured output prediction. Various capacity measures of function classes were used to study generalization bounds of MCC, including Rademacher complexities~\cite{lei2015multi,maurer2016vector,li2018multi,maximov2018rademacher,musayeva2019rademacher},
covering numbers~\cite{zhang2004statistical,lei2019data,AAAI_conv} and the fat-shattering dimension~\cite{guermeur2017lp}. While initial analyses implied generalization bounds with at least a linear dependency on the number of classes~\cite{koltchinskii2002empirical}, the couplings among class components were exploited recently to get a dependency that can be as mild as square-root~\cite{cortes2016structured,li2018multi} or even logarithmic~\cite{lei2019data,wu2021fine}.

\section{Problem Formulation\label{sec:formulation}}

%{\red the description is not quite clear to me enough. note the reviewers may not be familar with the strcutured prediction. We need to make sure that even those not familar with structured prediction can understand the problem. Please take a look of \cite{cortes2016structured} and think how they describe their problem set up. It would be beneficial to give some pictures to illustrate the notations. The logistic can also be improved. For example, the first paragraph gives some description of $L$. Then the third also give some information on $L$. We should present it compactly.}

% \begin{definition}
%   \end{definition}
% \begin{theorem}[Hello]
%   Hello
%   \end{theorem}
%The notation follows that of \cite{cortes2016structured}.
SOP refers to machine learning problems with an internal structure in the outputs (and potentially also in the inputs). For example in sequence-to-sequence prediction, both the input and output are sequences.  In syntax analysis, the inputs are sequences of words and the output is a parse tree.
%{\color{red} Here we mention examples with a structure in the inputs}

Let $\mathcal{X}$ be an input space (e.g., sentences in a given language) and $\Y$ be an output space (e.g., POS tags for the input sentences). In structured output prediction, the output space can often be decomposed into a number of substructures. Take POS tags as an example, where each word tag represents a substructure and the sequence of tags constitutes the structured output. Formally we define $\Y = \Y_1 \times \cdots \times \Y_l$, where $\Y_k$ is the set of possible classes a substructure $k$ can take. For a point $(x,y) \in \X \times \Y$, let $y^k$ denote the $k$-th element in $y$ (i.e., $y = (y^1, \ldots, y^l)$).

We aim to learn a scoring function $h:\X \times \Y \rightarrow \R$ based on which we can perform the prediction as $\hat{y}(x)  = \arg\max_{y \in \Y} h(x,y)$. The score function in structured output prediction can be described via a  factor graph $G = (V,F, E)$, where $V = [l]:=\{1, \ldots, l\}$ is the set of variable nodes, $F$ is a set of factor nodes, and $E$ is a set of undirected edges between a variable node and a factor node. Let $\Ne(f)$ be the set of nodes connected to the factor $f$ by an edge and $\Y_f = \Pi_{k\in \Ne(f)} \Y_k$. For brevity, we assume that $|\Y_f|=d$ for all $f$, where $|\Y_f|$ denotes the cardinality of $\Y_f$. Now we define the scoring function $h(x,y)$ for $x\in \X$ and $y \in \Y$ as
\[h(x,y) = \sum_{f \in F }h_f(x, y_f),\]
where $y_f:=\{y^j:j\in\Ne(f)\}$ and $h_f: \mathcal{X} \times \mathcal{Y}_f \rightarrow \mathbb{R}$. Figure~\ref{fig:factor_graphs} gives an example of factor graphs and scoring functions.
  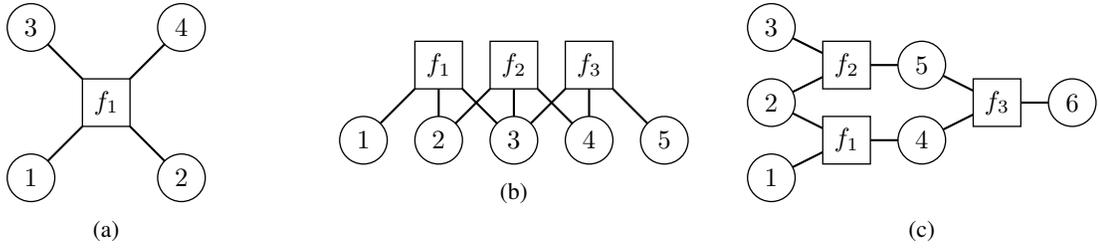
\begin{figure*}[h]
    \centering

    \begin{subfigure}{0.3\textwidth}
      \centering
      \begin{tikzpicture}
        % \tikzset{VertexStyle/.style = {shape = circle,fill = black,minimum size = 2pt}}
    \Vertex[x=1,y=1,L=$1$]{y1}
    \Vertex[x=3,y=1,L=$2$]{y2}
    \Vertex[x=1,y=3,L=$3$]{y3}
    \Vertex[x=3,y=3,L=$4$]{y4}
    % \Vertex[x=5,y=1,L=$y_5$]{y5}
    \begin{scope}[VertexStyle/.append style = {shape = rectangle}]
      \Vertex[x=2,y=2,L={$f_1$}]{f1}
\end{scope}

    \Edge(y1)(f1)
    \Edge(y2)(f1)
    \Edge(y3)(f1)
    \Edge(y4)(f1)

  \end{tikzpicture}
  \caption{}
\end{subfigure}
\begin{subfigure}{0.3\textwidth}
  \centering
  \begin{tikzpicture}

    \Vertex[x=1,y=1,L={$1$}]{y1}
    \Vertex[x=2,y=1,L={$2$}]{y2}
    \Vertex[x=3,y=1,L={$3$}]{y3}
    \Vertex[x=4,y=1,L={$4$}]{y4}
    \Vertex[x=5,y=1,L={$5$}]{y5}
    \begin{scope}[VertexStyle/.append style = {rectangle}]
      \Vertex[x=2,y=2,L={$f_1$}]{f1}
    \end{scope}
    \begin{scope}[VertexStyle/.append style = {rectangle}]
      \Vertex[x=3,y=2,L={$f_2$}]{f2}
    \end{scope}

    \begin{scope}[VertexStyle/.append style = {rectangle}]
      \Vertex[x=4,y=2,L={$f_3$}]{f3}
    \end{scope}

    % \begin{scope}[VertexStyle/.append style = {rectangle}]
    %   \Vertex[x=4,y=2,L={$f_4$}]{f4}
    % \end{scope}

    \Edge(y1)(f1)
    \Edge(y2)(f1)
    \Edge(y3)(f1)
    \Edge(y2)(f2)
    \Edge(y3)(f2)
    \Edge(y4)(f2)
    \Edge(y3)(f3)
    \Edge(y4)(f3)
    \Edge(y5)(f3)
  \end{tikzpicture}
  \caption{}
  \label{fig:sequence}
\end{subfigure}
\begin{subfigure}{0.30\textwidth}
      \centering
      \begin{tikzpicture}
        % \tikzset{VertexStyle/.style = {shape = circle,fill = black,minimum size = 2pt}}
    \Vertex[x=1,y=1,L=$1$]{y1}
    \Vertex[x=1,y=2,L=$2$]{y2}
    \Vertex[x=1,y=3,L=$3$]{y3}
    % \Vertex[x=3,y=3,L=$4$]{y4}
    \Vertex[x=3,y=1.5,L=$4$]{y4}
    \Vertex[x=3,y=2.5,L=$5$]{y5}
     \Vertex[x=5,y=2,L=$6$]{y6}
    % \Vertex[x=5,y=1,L=$y_5$]{y5}
    \begin{scope}[VertexStyle/.append style = {shape = rectangle}]
      \Vertex[x=2,y=1.5,L={$f_1$}]{f1}
\end{scope}
\begin{scope}[VertexStyle/.append style = {shape = rectangle}]
      \Vertex[x=2,y=2.5,L={$f_2$}]{f2}
\end{scope}
\begin{scope}[VertexStyle/.append style = {shape = rectangle}]
      \Vertex[x=4,y=2,L={$f_3$}]{f3}
\end{scope}

     \Edge(y1)(f1)
    \Edge(y2)(f1)
    \Edge(y3)(f2)
    \Edge(y2)(f2)
    \Edge(y5)(f2)
    \Edge(y4)(f1)
     \Edge(y5)(f3)
      \Edge(y4)(f3)
       \Edge(y6)(f3)

  \end{tikzpicture}
  \caption{}
\end{subfigure}

\caption{Examples of factor graphs.  Panel (a) represents a factor graph with only one factor node. Note that $\Ne(f_1)=\{1,2,3,4\}$ and $\Y_{f_1}=\Y_1\times\Y_2\times\Y_3\times\Y_4$. If $\Y_i=\{1,2,3\}$ for all $i$, then $|\Y_{f_1}|=3^4$. The corresponding scoring function is  $h(x,y) =  h_{f_1}(x,y^1,y^2,y^3,y^4)$. Panel (b) depicts an example of factor graph that assumes a sequence-like structure. The scoring function in this case is $h(x,y) = h_{f_1}(x,y^1,y^2,y^3) + h_{f_2}(x,y^2,y^3,y^4) + h_{f_3}(x,y^3,y^4,y^5)$. Panel (c) depicts an example of tree-like factor graph. }
\label{fig:factor_graphs}
\end{figure*}

Let  $S = \{(x_i,y_i)\}_{i=1}^m$ be a training set with $(x_i,y_i) \in \X \times \Y$ being independently drawn from a distribution $\D$ over $\X\times\Y$.  We use a loss function $L: \Y \times \Y \rightarrow \R_+$ to measure the performance of prediction models, based on which we can define the margin loss \cite{cortes2016structured} as $L_\rho: \X \times \Y \times \F \rightarrow \R$:
  \begin{equation}
    \label{eq:loss}
    L_{\rho}(x,y,h) = \Phi^*(\max_{y' \ne y} \{ L(y',y) - \frac{1}{\rho}[h(x,y) - h(x,y')]\}),
\end{equation}
where $\Phi^*(r) = \min(M,\max(0,r))$, $M = \max_{y,y'}L(y,y')$ and $\F \subset \{h: \X \times \Y \rightarrow \R\}$ is some hypothesis class. Note that $L_\rho(x,y,h) \ge L(\hat{y}(x),y)$. Therefore, the obtained bounds for $L_\rho$ will also hold for $L$. We then define the population risk $R(h)$ and empirical risk $R_S(h)$ to quantify the performance of a model $h$ on testing and training examples, respectively as:
%{\color{red}Yunwen: y(x) should be h(x)?} : Here we use the loss at the classifier output which defined above as \hat{y}(x)
\[
    R(h) = \E_{\D}[L_\rho(x,y,h)],\quad
    R_S(h) = \frac{1}{m}\sum_{i = 1}^mL_\rho(x_i,y_i,h).
\]
Let $\Psi$ be a feature function which maps an input-output example $(x,y) \in \X \times \Y$ to $\R^D$, where $D$ is the dimension of feature space. In structured output prediction, the feature extractor takes a composite form
according to the factor graph $G$, that is, $\Psi(x,y) = \sum_{f \in F} \Psi_f(x,y_f)$, where $\Phi_f: \mathcal{X} \times \mathcal{Y}_f \rightarrow \mathbb{R}$. We consider a linear scoring function $h^w(x,y) = \left<w,\Psi(x,y)\right>$ indexed by a $w \in \R^D$. Then the hypothesis space becomes
\begin{equation}
  \label{eq:linear_class}
  \F_p = \big\{ (x,y) \mapsto \left<w, \Psi(x,y) \right> : \|w\|_p \le \Lambda , (x,y) \in \X  \times \Y\big\},
\end{equation}
where $\|w\|_p = (\sum_{i=1}^D |w_i|^d)^{\frac{1}{p}}$ is the $\ell_p$-norm of $w=(w_1,\ldots,w_D)$.
We also define the class of loss functions
\begin{equation}
  \label{eq:loss_class}
F_{p,\Lambda,\rho} := \big\{(x,y) \mapsto L_{\rho}(x,y,h^w): h^w \in \F_p\big\}.
\end{equation}

% Define the following set based on the training sample $S$.
% \begin{equation*}
% \widetilde{\F}_{p,\Lambda} := \{v \mapsto \left<w,v\right>: w \in \R^D,\|w\|_p\le \Lambda, v \in \widetilde{S}\},
% \end{equation*}
% where $\widetilde{S}$ is defined as follows,
% \begin{equation*}
%   \widetilde{S} := \{ \Psi_f(x,y): x \in S_{|\X}, y \in \Y_f , f \in F\}.
% \end{equation*}
% That is, the set of size $md|F|$ constructed by the following. For each input $x$ in the original training set and each feature function $\Psi_f(.,.)$, we construct $|\Y_f|$ training examples as $\Psi_f(x,y')$, for all $y' \in \Y$.

% Similarly define
% \begin{equation*}
% \widetilde{\F}'_{p,\Lambda} := \{v \mapsto \left<w,v\right>: w \in \R^D,\|w\|_p\le \Lambda, v \in \widetilde{S}'\},
% \end{equation*}
%  where,
% \begin{equation*}
%   \widetilde{S}' := \{ \Psi_f(x,y_f): (x,y) \in S , f \in F\}.
% \end{equation*}

% That is, the set of size $m|F|$ constructed similarly but with only the correct labels from the training set.

\section{Main Results\label{sec:result}}
% Now that the required notation is introduced, we prove our results. The main goal is to derive a bound on the Rademacher complexity of the class defined in equation \eqref{eq:loss_class} in terms of the $\ell_\infty$-covering number of the class defined in equation \eqref{eq:linear_class}. For that analysis the loss \eqref{eq:loss} is required to be $\ell_\infty$-Lipschitz. It is however not. On the other hand, the Rademacher complexity of the \eqref{eq:loss_class} can be upper bounded by two Rademacher complexities of the classes $F^1_{p,\Lambda,\rho}$ and $F^2_{p,\Lambda,\rho}$, defined below, which are defined by the composition of the hypothesis class $\F_p$ and $\ell_\infty$-Lipschitz losses. Therefore, it suffices to upper bound Rademacher complexities of $F^1_{p,\Lambda,\rho}$ and $F^2_{p,\Lambda,\rho}$ by the corresponding covering numbers. Further, we notice that the $\ell_\infty$-covering numbers of $F^1_{p,\Lambda,\rho}$ can be controlled by the $\ell_\infty$-covering number of the class $\widetilde{\F}_{p,\Lambda}$ and that the covering number of $F^2_{\rho,\Lambda}$ is upper bounded by the covering number of $\widetilde{\F}'_{p,\Lambda}$.

In this section, we present our main results on generalization bounds for structured output prediction. 
% We consider two types of generalization bounds: bounds with high probability and bounds in expectation. 
We consider two types of generalization bounds: complexity-based bounds and stability-based bounds.
Our aim is to develop bounds with a very mild dependency on the size of the label set, thus laying a solid foundation for structured output prediction, where the size of label set $\Y$ is often extremely large in practice. A key discovery to both our stability-based and complexity-based analysis is to note the Lipschitz continuity of loss functions w.r.t. infinity-norm $\|\cdot\|_\infty$.
%A key property we use in relating the covering numbers of $F_{p,\Lambda,\rho},S)$ to that of $\widetilde{\F}_{p,\Lambda}$ is the Lipschitz continuity of loss functions~\cite{lei2019data} in the sense of infinity-norm $\|\cdot\|_\infty$.
\begin{definition}[Lipschitz continuity]
    We say that a loss function $L(x,y,h)$ is $(\tau,\ell_\infty)$-Lipschitz in the last argument if, for any $h,\tilde{h} \in \F$ and all $(x,y) \in \X \times \Y$, we have:
 $$
 |L(x,y,h) - L(x,y,\tilde{h})| \le \tau \max_{y'\in \Y}|h(x,y') - \tilde{h}(x,y')|.
 $$
  \end{definition}
The existing analysis \cite{cortes2016structured} uses the $(\tau_2,\ell_2)$ Lipschitz continuity of loss functions:
\[|L(x,y,h) - L(x,y,\tilde{h})|\le \tau_2\Big(\sum_{y'\in \Y}|h(x,y') - \tilde{h}(x,y')|^2\Big)^{1/2}.
\]
Note that the Lipschitz continuity w.r.t. $\ell_\infty$-norm is much stronger than that w.r.t. $\ell_2$-norm. Indeed, if $L$ is $(\tau,\ell_\infty)$-Lipschitz then it is also $(\tau,\ell_2)$ Lipschitz since $\|\cdot\|_\infty\leq\|\cdot\|_2$. As a comparison, a $(\tau_2,\ell_2)$-Lipschitz function can be $(\tau_2\sqrt{|\Y|},\ell_\infty)$-Lipschitz due to the norm relationship $\|\cdot\|_2\leq\sqrt{|\Y|}\|\cdot\|_\infty$ (the equality can hold in some cases).

In Lemma \ref{lem:lipschitzness} we build the $\ell_\infty$-Lipschitz continuity of $L_\rho$ for structured output prediction. A remarkable property is that the involved Lipschitz constant is independent of $|\Y|$. This shows that the loss function in structured output prediction is well behaved in the sense of Lipschitz continuity. However, the existing analysis based on the $(\tau_2,\ell_2)$-Lipschitz continuity fails to exploit this strong regularity, and therefore only implies suboptimal bounds with at least a square-root dependency on the size of the label set.
The proof of Lemma~\ref{lem:lipschitzness} below is given in the appendix. %deferred to the appendix.  %{\red When stating a lemma or a theorem we need to indicate clearly where the proof is}
\begin{lemma}\label{lem:lipschitzness}
The loss function $L_\rho$ is $(\frac{2}{\rho},\ell_\infty)$-Lipschitz with respect to the scoring function $h$ for all $x \in \X$ and $y \in \Y$.
\end{lemma}

% \subsection{Generalization Bounds with High Probability}
\subsection{Complexity-based Generalization Bounds}
We develop generalization bounds with high probability here. Our basic tool to this aim is the Rademacher complexity. %. We present our generalization bound to the loss class defined in equation \eqref{eq:loss_class}. We approach the problem by bounding the Rademacher complexity of the loss class \eqref{eq:loss_class}.
%First we give the definition of the Rademacher complexity.
\begin{definition}
    The empirical Rademacher complexity of a function class $\F$ of real-valued functions is defined as:
    \begin{equation}
      \label{eq:Rademacher_sep}
      \Rad_S(\F) = \E_{\sigma}[\sup_{f \in \F}\frac{1}{n} \sum_{i = 1}^n\sigma_i f(x_i)],
    \end{equation}
    where $\{\sigma_i\}$ are random variables with equal probability of being either $+1$ or $-1$.
  \end{definition}
%   We now give the definition of $\ell_\infty$-covering numbers, which serves as our main analysis tool.
%   \begin{definition}
%     Let $\F$ be a class of real-valued functions defined over a space $Z$ and $S:=\{z_1, \ldots, z_n\}  \subset Z$. For any $\epsilon > 0$, the empirical $\ell_\infty$-covering number denoted by $\N(\epsilon,\F,S)$ with respect to $S$ is defined as the minimal number $m$ of collection of vectors $v^1,\ldots, v^m \in \R^n$, called a cover, such that
%     \begin{equation*}
%       \sup_{h \in \F}\min_{j=1,\ldots,m}\max_{i=1, \ldots, n}|h(z_i)- v_i^j| \le \epsilon.
%     \end{equation*}

%   \end{definition}
  \begin{theorem}\label{thm:main}
  Let $\rho > 0$ be. Then the Rademacher complexity of the loss class $F_{p,\Lambda,\rho}$ is bounded as follows:
 \begin{equation}
  \Rad(F_{p,\Lambda,\rho}) \le \frac{4}{m} + \frac{144 \sqrt{q-1}\Psi^* \Lambda |F|}{\rho\sqrt{m}}  \tilde{L},
\end{equation}
where $\tilde{L} = \sqrt{\log(2md|F|[8\Psi^* \Lambda m|F|/\rho + 3]+1)}\log(m)$, $\Psi^* = \sup_{f\in F,y \in \Y_f,x\in \X} \|\Psi_f(x,y)\|_q$,
and $q=p/(p-1)$. % is defined such that $\frac{1}{q} + \frac{1}{p} = 1$.
 \end{theorem}
 The proof strategy is to 
 %utilize the $\ell_\infty$-Lipschitzness of the loss
 to relate the complexity of the loss class $F_{p,\Lambda,\rho}$ to a complexity of a scalar linear function class on an extended set of size $m|F|d$, thus moving contribution of $d$ to the complexity from the output dimension to the size of training set. We then utilize standard bounds \cite{zhang2002covering} that admit log dependency on the size of training set. The detailed proof is given in the appendix.
% {\red what is $q$?}
%{\red I suggest not to include the comparison in a subsection. This can be inclquded in a remark. Note that we need to present another bound on generalization bounds other than Rademacher complexity. Generalization is the quantity we are interested in.}
%   {\red This subsection should be presented after the main results}
\begin{remark}
    We now compare our results with related work. In \cite{cortes2016structured}, the authors bounded the Rademacher complexity of $F_{p, \Lambda,\rho}$ by a factor graph Rademacher complexity. Specifically for the loss class \eqref{eq:loss_class} they proved
   $
     \Rad(F_{p,\Lambda,\rho}) \le \frac{2\sqrt{2}}{\rho} \hat{\Rad}^G_S(\F_p),
   $
   where $\hat{\Rad}^G_S(\F_p)$ is defined as
\begin{equation*}
   \frac{1}{m} \E_{{ \bf \epsilon}}\Big[\sup_{h \in \F} \sum_{i \in [m], f \in F,y \in \Y_f} \sqrt{|F|} \epsilon_{i,f,y} \left<w,\Psi_f(x_i,y)\right>\Big].
 \end{equation*}
 Here $\epsilon = (\epsilon_{i,f,y})_{i \in [m],f \in F, y \in \Y_f}$ and each $\epsilon_{i,f,y}$ is an independent Rademacher variable. Combining the result from Theorem 2 in \cite{cortes2016structured}, we get the following bound for learning with $\ell_2$-regularization:
   $\Rad(F_{2,\Lambda,\rho}) \le O\Big(\frac{\Lambda\Psi^*|F|\sqrt{d}}{\rho\sqrt{m}}\Big)$.
Note the bound exhibit a square-root dependence on the number of classes per factor $d = |\Y_f|$. Thus it is vacuous for typical SOPPs, where the number of class labels grows exponentially w.r.t. the size of the output. For comparison, our bounds enjoy a log dependency on $d$ and therefore still imply meaningful generalization bounds in this setting.
\end{remark}
%{\red I suggest to not use a single subsection for applications. We can put some single corollaries. Corollary 1 is on Markov networks and then we give remarks on the comparison with existing work. This would make the paper better organized. I think we do not need to consider the application to multi-class classification. The reason is that the multi-class classification bound in Cortes 2016 is not the state of the art}

 As a direct corollary, we use the connection between generalization and Rademacher complexity to get Theorem \ref{thm:gen-hp}.
\begin{theorem}[Generalization Bounds\label{thm:gen-hp}]  For any $\rho > 0 $, $\delta\in(0,1)$, and $h \in \F_p$, with probability at least $1 - \delta$ over the draw of training data $S$,  the following  bound holds:
%\begin{equation*}
%\begin{split}
\[    R(h ) \le R_S(h) + \frac{8}{m}   + \frac{288\sqrt{q-1}\Psi^* \Lambda |F|}{\rho\sqrt{m}}  \tilde{L}
    + M\sqrt{\frac{\log{\frac{1}{\delta}}}{2m}}.\]
%    \end{split}
%\end{equation*}
%where $R^\rho_S(h) = \frac{1}{m} \sum_{i = 1}^m {L_\rho(x_i,y_i,h)}$.%{\color{red} Yunwen: here the notation should be consistent. Note we have defined $\hat{R}$ in page 2. I would suggest to use the notation $R_S$ as it is frequently used in the paper. We then need to give the definition of $R_S$ in page 2} $R^\rho_S(h)$ is only used here, the other mentions of R_S refer to the original loss not the surrogate 
\end{theorem}

\subsection{Stability-based Generalization Bounds}
In this section, we present generalization bounds in expectation for structured output prediction by leveraging the lens of algorithmic stability. Algorithmic stability is a fundamental concept in statistical learning theory, which measures the sensitiveness of output models when the training dataset of an algorithm $A$ is slightly perturbed.
For any algorithm $A$, we use $A(S)$ to denote the model produced by running $A$ over the training examples $S$.
\begin{definition}[Uniform Stability\label{def:unif-stab}]
  A stochastic algorithm $A$ is $\epsilon$-uniformly stable if, for all training datasets $S,\widetilde{S}\in\mathcal{Z}^n$ that differ by at most one example, we have
  \begin{equation}\label{unif-stab}
  \sup_{x,y}\mathbb{E}_A\big[L_\rho(x,y,A(S))-L_\rho(x,y,A(\widetilde{S}))\big]\leq\epsilon.
  \end{equation}
\end{definition}
Algorithmic stability naturally implies quantitative generalization bounds, as shown in the following lemma~\cite{shalev2010learnability}.
\begin{lemma}[Generalization via uniform stability\label{lem:gen-stab}]
  Let $A$ be $\epsilon$-uniformly stable. Then
  $
  \big|\mathbb{E}_{S,A}\big[R(A(S))-R_S(A(S))\big]\big|\leq\epsilon.
  $
\end{lemma}
We will apply algorithmic stability to study two representative algorithms for SOP: regularization and stochastic gradient descent.
For brevity, we use the abbreviation $R(w)=R(h^w), R_S(w)=R_S(h^w)$, etc. We also write $w^*=\arg\inf_w R(w)$ for the minimizer of the population risk.

\textbf{Regularized Risk Minimization}. RRM is a popular scheme to overcome overfitting in machine learning. The basic idea is to add a regularizer to the empirical risk and build a regularized empirical risk $R_S^\lambda$. Then we minimize the resulting objective function to obtain a model $w_S$ as follows:
\begin{equation}\label{regularization}
w_S=\arg\min_{w}\big[R_S^\lambda(h^w):=R_S(h^w)+\frac{\lambda}{2}\|w\|_2^2\big].
\end{equation}
Here we omit the dependency of $w_S$ on the regularization parameter for brevity.
In the following lemma to be proved in the appendix, we show that the above regularization algorithm is uniformly stable. Let $\kappa:=\sup_{x,y}\|\Psi(x,y)\|_2$.
\begin{lemma}\label{lem:stab-reg}
  Let $A$ be defined as \eqref{regularization}, i.e., $A(S)=w_S$. Then $A$ is $\frac{16\kappa^2}{m\rho^2\lambda}$-uniformly stable.
\end{lemma}
This lemma is a variant of the stability bound in \cite{bousquet2002stability}, which, however, requires the loss function to be admissible. We adapt their technique to the setting of structured output prediction and a key step in our analysis is again the Lipschitz continuity of the loss function w.r.t. the $\ell_\infty$ norm. A use of the classical Lipschitz continuity w.r.t. $\ell_2$ norm would incur a bound with at least a square-root dependency on $d$. For comparison, the consideration of Lipschitz continuity w.r.t. the $\ell_\infty$ norm allows us to get stability bounds independent of the size of the label set.

We can combine the Lipschitz continuity of loss functions, the stability of regularization schemes established in Lemma \ref{lem:stab-reg} and Lemma \ref{lem:gen-stab} together to get the following generalization bounds for structured output prediction. Let \[w^\lambda=\arg\inf_w \big[R^\lambda(w):=R(h^w)+\frac{\lambda}{2}\|w\|_2^2\big]\] be the minimizer of the regularized risk. We have the following result, whose proof is given in the appendix.
\begin{theorem}\label{thm:reg}
  Let $w_S$ be defined in \eqref{regularization}. Then
    \begin{equation}\label{reg-a}
  \mathbb{E}\big[R^\lambda(w_S)-R^\lambda(w^\lambda)\big]\leq \frac{16\kappa^2}{m\rho^2\lambda}.
  \end{equation}
  Furthermore, if we choose $\lambda=\frac{4\sqrt{2}\kappa}{\sqrt{m}\rho \|w^*\|_2}$, then
  \begin{equation}\label{reg-b}
  \mathbb{E}\big[R(w_S)\big]-R(w^*)\leq \frac{4\sqrt{2}\kappa\|w^*\|_2}{\sqrt{m}\rho}.
  \end{equation}
\end{theorem}

\textbf{Stochastic Gradient Descent}. We now turn to the performance of SGD for structured output prediction. SGD is a popular optimization algorithm with wide applications in learning in a big data setting. Let $w^{(1)}$ be the initial point and $\{\eta_t\}$ be a sequence of positive step sizes. At the $t$-th iteration, we first randomly select an index $i_t$ according to the uniform distribution over $[m]$, which is used to build a stochastic gradient $L_\rho'(x_{i_t},y_{i_t},h^{w^{(t)}})$ ($L_\rho'(x_{i_t},y_{i_t},h^{w^{(t)}})$ denotes a subgradient of $L_\rho(x_{i_t},y_{i_t},h^{w})$ at $w=w^{(t)}$). Then we update the model along the negative direction of the stochastic gradient
\begin{equation}\label{sgd}
w^{(t+1)}=w^{(t)}-\eta_tL_\rho'(x_{i_t},y_{i_t},h^{w^{(t)}}).
\end{equation}
This scheme of selecting a single example to build a stochastic gradient allows SGD to get sample-size independent iteration complexity, and is especially appealing if $m$ is large. Since we consider a linear scoring function $h^w$, the loss function $L_\rho$ is convex w.r.t. $w$. In the following lemma to be proved in the appendix, we build the uniform stability of SGD for structured output prediction. Note here we do not require the loss function to be smooth~\cite{lei2020fine}.
\begin{lemma}\label{lem:stab-sgd}
Let $S=\{z_1,\ldots,z_m\}$ and $S'=\{z_1',\ldots,z_m'\}$ be two datasets that differ only by a single example. Let $\{w^{(t)}\}$ and $\hat{w}^{(t)}$ be two sequences produced by SGD based on $S$ and $S'$, respectively. Then
\[
\mathbb{E}_{A}\big[\|w^{(t+1)}-\hat{w}^{(t+1)}\|_2^2\big]\leq 16e(1+t/m^2)\kappa^2\rho^{-2}\sum_{j=1}^{t}\eta_j^2.
\]
\end{lemma}
According to Lemma \ref{lem:stab-sgd}, we know that the algorithm becomes more and more unstable as we run more and more iterations. We can use this stability bound to derive generalization bounds of SGD for structured output prediction. The proof is given in the appendix.
\begin{theorem}\label{thm:gen-sgd}
Let $\!\{w^{(t)}\}$ be produced by \eqref{sgd} with $\eta_t\!=\!\eta$. Then
\begin{equation}\label{gen-sgd-a}
\mathbb{E}\big[R(\bar{w}^{(T)})\big]-R(w^*)\leq O\Big((\sqrt{T}+T/m)\kappa^2\eta+\frac{1\!+\!T\kappa^2\eta^2}{T\eta}\Big),
\end{equation}
where $\bar{w}^{(T)}=\frac{1}{T}\sum_{t=1}^{T}w^{(t)}$ is an average of iterates.
\end{theorem}
The upper bound \eqref{gen-sgd-a} involves two terms. The first term $\sqrt{T}+T/m$ comes from controlling the generalization error $R(\bar{w}^{(T)})-R_S(\bar{w}^{(T)})$, while the second term $\frac{1+T\eta^2}{T\eta}$ comes from controlling the optimization error $R_S(\bar{w}^{(T)})-R_S(w^*)$. It is clear the optimization error decreases w.r.t. $T$, while the generalization error grows in the learning process. Therefore, we need to trade-off these two terms by early-stoping SGD as done by the following corollary.
We write $B\asymp \widetilde{B}$ if there are absolute constants $c_1$ and $c_2$ such that $c_1B\leq \widetilde{B}\leq c_2B$.
\begin{corollary}\label{cor:gen-sgd}
Let $\{w^{(t)}\}$ be the sequence produced by \eqref{sgd} with $\eta_t=\eta$. If we choose $T\asymp m^2$ and $\eta\asymp T^{-\frac{3}{4}}/\kappa$, then
\[\mathbb{E}\big[R(\bar{w}^{(T)})\big]-R(w^*)\leq O(\kappa m^{-\frac{1}{2}}).\]
\end{corollary}
%\begin{proof}
%  For the choice $T\asymp m^2$ and $\eta\asymp T^{-\frac{3}{4}}/\kappa$, it is clear
%  \[
%  (\sqrt{T}+T/m)\kappa^2 \eta\leq O(\kappa m^{-1/2}),\; \frac{1+T\eta^2\kappa^2}{T\eta}\leq O(\kappa m^{-1/2}).
%  \]
%  The stated result then follows from Theorem \ref{thm:gen-sgd}.
%\end{proof} 

\begin{remark}
According to Theorem \ref{thm:reg} and Corollary \ref{cor:gen-sgd}, we know that both the regularization method and SGD are able to achieve the generalization bound $O(1/\sqrt{m})$, which is minimax optimal. While RRM requires the objective function to be strongly convex, SGD only requires the objective function to be convex. Remarkably, these generalization bounds do not admit any dependency on the size of the label set, and provide a convincing explanation on why SOP often works well even if the problem has more class labels than training examples. To our best knowledge, these are the first label-size free generalization bounds. As compared to Theorem \ref{thm:gen-hp} on high-probability bounds, our generalization bounds here are stated in expectation. It should be noted that our bounds in expectation require the loss functions to be convex, while the high-probability analysis also applies to nonconvex cases.
\end{remark}

\subsection{Applications}

In this section we discuss applications of our bounds and compare them to those of \cite{cortes2016structured}.

\begin{example}\label{exp:mn}Consider pair-wise Markov networks with fixed number of substructures $l$ \cite{taskar2003max}. Specifically, we have $\Y = \Y_1 \times \ldots \times \Y_l$ and $\Y_k \in [c]$ for $k \in [l]$. Further, we have sequence-like connections, i.e., there is an arrangement of output nodes such that if a factor $f \in F$ is connected to two nodes then they are neighbors in that arrangement. Therefore we have $|F|  = l-1$ and $d =c^2$. We further assume an unnormalized hamming loss $L(y, y') = \sum_{k = 1}^l \mathbb{I}_{y_k \ne y'_k}$ so that we normalize later in the bound to get rid of the dependence on $l =|F| +1$.
For regularized learning with these Markov networks, the Rademacher complexity of loss function classes was bounded in \cite{cortes2016structured}
\begin{equation*}
  \Rad(F_{2,\Lambda,\rho}) \le O(\frac{\Lambda\Psi^* c}{\rho \sqrt{m}}).
\end{equation*}
As a comparison, our Rademacher complexity bound in Theorem \ref{thm:main} reduces to an upper bound on $\Rad(F_{2,\Lambda,\rho})$ that has the form
\[
 O\left( \frac{ \Lambda \Psi^*\log m \sqrt{\log(2mc^2l[8\Psi^* \Lambda m/\rho \!+\! 3]\!+\!1)}}{\rho\sqrt{m}}\right).
\]
Therefore, our bound significantly outperforms their bound by dropping their linear dependency on $c$ to a logarithmic dependency. If we further extend the model so that each factor $f$ is connected to $v$ nodes instead of 2, their bound grows, as a function of $v$, as $O(c^{v/2})$ while ours increase only $O(\sqrt{v})$. Furthermore, according to Theorems \ref{thm:reg}, \ref{thm:gen-sgd}, we can get generalization bounds $O(\kappa/\sqrt{m})$ in expectation for both RRM and SGD, where the log dependency 
%in our high-probability bounds 
is further removed. %Note that $\kappa\leq |F|\Psi^*$.
% only log dependence on the number of classes $|\Y_k|=c$. However, the bound from \cite{cortes2016structured} has a linear dependence on the number of classes a substructure can take.

\end{example}
\begin{example}
As the second example we consider multi-class classification. In this case we have no substructures and therefore $|F| = 1, \Y_1   = \Y$ where $\Y = [c]$, $d = c$.
In \cite{cortes2016structured}, the Rademacher complexity for multi-class learning with $\ell_2$ regularization was shown to satisfy
\begin{equation*}
  \Rad(F_{2,\Lambda,\rho}) \le O\left(\frac{\Psi^* \Lambda  \sqrt{c}}{\rho\sqrt{m}}\right).
\end{equation*}
Our analysis instead shows $\Rad(F_{2,\Lambda,\rho})$ is bounded by
\begin{equation*}
   O\left( \frac{\Psi^* \Lambda\sqrt{\log(2mc[8\Psi^* \Lambda m/\rho\! +\! 3]\!+\!1)} \log m}{\rho\sqrt{m}}  \right).
\end{equation*}
It is clear that we drop the square root dependency in $c$ in \cite{cortes2016structured} to a log dependence. 
Analogous to Example \ref{exp:mn}, the log dependency can be further removed if we consider generalization bounds in expectation, as shown in Theorems \ref{thm:reg} and \ref{thm:gen-sgd}.
%Note that the authors of \cite{cortes2016structured} proved a bound with a log dependence on the number of classes under a $\ell_{2,1}$ regularization on the weights. We note that if all classes weights have almost the same $\ell_2$-norm, which is a realistic assumption, then the $\ell_{2,1}$-norm is $\sqrt{c}$ larger than the $\ell_{2}$ norm and therefore their bound also increases as the square root of the number of classes.
\end{example}
\begin{example} In this example we explore the possibility of combining SOP models above with a learned feature extraction function $\Psi$ as was practically explored in \cite{chen2017deeplab,hinton2012deep}.
Consider the case where $\Psi$ is a CNN that takes $x$ as input and outputs different $D$-dimensional vector $\Psi_f(x,y_f)$ for each factor $f$ and label $y_f$. Chaining the covers, one can bound the Rademacher complexity of the combined class as follows:
\[\widetilde{O}\left(\frac{\sqrt{q-1}\Psi^* \Lambda |F|}{\rho\sqrt{m}}\right)+\widetilde{O}\left(\sqrt{\frac{\bar{D}}{m}} \log(\tilde{G})\right), \]
where the notation $\widetilde{O}$ hides logarithmic factors, $\bar{D}$ is the number of network parameters and $\tilde{G}$ is a product of norms of network weight matrices. The details of the bound and its derivation of this bound are given in the appendix. 

\end{example}

\section{Learning  Weakly Dependent Sequences\label{sec:dependent}}

In the above bounds we assumed that the examples are sampled independently from each other. However, this assumption is often violated in practice. For example, consider the problem of POS tagging. We are usually given a dataset of documents each of which contains a sequence of sentences. There are two natural assumptions. (1) We may assume that each document is a long sequence of dependent words. This assumption is too pessimistic. The considered sample size becomes too small, and the prediction complexity increases while, as sentences get further apart, the dependence between them decreases, and thus the effective sample size increases. (2) We may assume that each sentence is independent of the others within and across documents. This assumption on the other hand is too optimistic. Sentences following each other in the same document indeed have some degree of dependence. We formalize this dependence in a hierarchical manner, thus providing a trade-off between these two assumptions. Namely, we assume that the documents are independent of each other while sentences within a document are only weakly dependent. We note that the term document here does not necessarily mean an actual text document but rather any sequence of examples (e.g., for a dataset of videos, one video is a document as it is a sequence of images).

We now formalize the idea above. We are given a training set of independent documents $\{D_i\}_{i=1}^m$. Each document $D_i$ is a sequence of weakly dependent examples $D_i = (z^j_i)_{j =1 }^J$. Since the structured output prediction framework in the above section subsumed the usual classification paradigm, we assume that the sequence elements classes follows it. That is, $z^i_j \in \X \times \Y =: \Z$, where $\X$ and $\Y$ defined as above.
% and the loss $L: \Y \times \Y \rightarrow \R_{+}$ is decomposed as $L(Y,Y') = \sum_{k = 1}^l L_k(y_k,y'_k) $

 Now we define precisely how the examples within each document $D_i$ are weakly dependent. We assume that each example within a given document is sampled from a $\beta$-mixing process, defined below, at times $1, 2, \ldots, J$.

\begin{definition}[Stationary $\beta$-mixing Stochastic Process]

  Let $(z^k)_{k=-\infty}^\infty$  be a stationary stochastic process and $\sigma_L = \sigma((z^k)_{k=1}^L)$ and $\sigma_{L+a} = \sigma((z^k)_{k=L+a}^\infty)$ be the sigma algebras generated by the random variables $Z_1^L=(z^1, \ldots, z^L)$ and $z_{L+a}^\infty = (z^{L+a}, \ldots, Z^\infty)$. Define the $\beta$-mixing coefficient
  $
    \beta(a) := \sup_{L \ge 1} \E \left[\sup_{B \in \sigma_{L+a}} |P(B|\sigma_L) - P(B)|\right].
  $

\noindent  The process is called $\beta$-mixing if $\lim_{a \rightarrow \infty}\beta(a) = 0$. It is called exponentially mixing if $\beta(a) \le \beta_0 \exp{(-\beta_1a^r)}$ or algebraically mixing if $\beta(a) \le \beta_0 / a^r$, for positive $\beta_0$, $\beta_1$ and~$r$.
\end{definition}
Some examples of exponentially mixing process include a class of Autoregressive Moving Average (ARMA) \cite{mokkadem1988mixing} and a class of Markov process \cite{rosenblatt2012markov}.
% We are given a training set $S= \{D_i\}_{i=1}^m$ obtained by sampling i.i.d from $P^l$ and a test example $Z$ sampled independently of $S$.

%As above we are interested in bounding the quantity $| R(h) - \hat{R}(h)|$ for all $h$ in a some function space $\F$. Note that the bounds above are suitable candidate bounds nevertheless, they do not exploit a weak dependence structure within each sequence. We derive bounds under the weakly dependent sequences by employing a famous technique in dealing with weakly dependent data, namely \textit{Independent Blocks} \cite{yu1994rates,mohry_stab_mixing2007}. Unlike \cite{yu1994rates,mohry_stab_mixing2007} here we assume that the sequences are independent of each other and the weak dependence is only within each sequence.

Now let $\F$ be a structured output prediction function class as defined in the previous section (see \eqref{eq:linear_class}). For $h \in \F$, let $L_h: \Z \rightarrow [0,M]$ be a loss function over elements of the sequence $z^k$. An example for such a loss is given in equation \eqref{eq:loss}. Again we are interested in high probability bounds on the difference between two quantities: the empirical risk $\hat{R}_S(h)$ and the true risk $R(h)$, which are defined as
\[
    R_S(h) = \frac{1}{mJ}\sum_{i=1}^m \sum_{j=1}^JL_h(z^j_i),\quad
    R(h) = \E_{S}[R_S(h)].
\]

Theorem \ref{thm:weakly} summarizes the main results of this section. %{\color{red}Yunwen: the constraint $\delta > 2m(\frac{J}{2a}-1)\beta(a)$ may be very restrictive? Note $\delta$ should be smaller than one meanwhile}
  %Our main result of this section is summarized by the following theorem:
      \begin{theorem}
      \label{thm:weakly}
      Let $F_{p,\Lambda,\rho}$ be the loss class defined in \eqref{eq:loss_class} and let $S$ be a set of independent and identically distributed documents $D_i$, $i = 1, \ldots, m$, where each document is a sequence of examples $(z_i^j)$ $j = 1, \ldots, J$ drawn from a $\beta$-mixing process. For any integer $a > 0$ such that $J$ is a multiple of $2a$. Let $\delta > 2m(\frac{J}{2a}-1)\beta(a)$, then with probability at least $1 - \delta$, the following inequality holds uniformly over all $h \in \F_p$
    \begin{equation}
    \begin{split}
      R(h) \le &\hat{R}_S(h) +  O \left(\frac{\sqrt{2a(q-1)}\Psi^* \Lambda |F|}{\rho\sqrt{mJ}}
      \tilde{L}\right)\\
      &+ \frac{M\sqrt{a}\sqrt{\log\left(\frac{2}{\delta- 2m(\frac{J}{2a}-1)\beta(a)}\right)}}{\sqrt{ mJ}},
      \end{split}
    \end{equation}
where $\tilde{L} = \sqrt{\log(2mJd|F|[8\Psi^* \Lambda mJ/\rho + 3]+1)} \log(mJ)$.
    \end{theorem}
% {\color{red}: we require a discussion of the result here. Explain how the result depend on the weak dependency and whether we can recover the previous results if there is no dependency}
\begin{remark}
Note that the bound unsurprisingly depends on the same main quantities as the bound in Theorem \ref{thm:gen-hp}. To better interpret it consider the following two extreme cases. (1) The elements inside each document are independent of each other. Note that in this case $\beta(a) = 0 $, for all $a$, hence $a$ can be chosen to be $1$ and the bound boils down to the bound in Theorem \ref{thm:gen-hp}. (2) The elements inside each document are strongly dependent. Thus, $\beta(a) \approx 1 $ for all $a$ and therefore selecting $a = \frac{J}{2}$ leads to the bound in Theorem~\ref{thm:gen-hp} with only $m$ training examples. We further note that $\beta(a) \rightarrow 0$ as $a \rightarrow \infty$, therefore, for any process admitting a fast decaying $\beta(a)$ the term $2m(\frac{J}{2a} -1) \beta(a)$ approaches $0$ fast for moderate $a$. 
\end{remark}

\section{Conclusion\label{sec:conclusion}}
In this paper, we advance the state of the art in the generalization analysis of structured output prediction. We consider two types of generalization bounds: complexity-based and stability-based bounds. Our complexity-based approach produces bounds with high probability that admit a log dependency on the size of the label set. The stability-based approach further removes this log dependency for generalization bounds in expectation.
%through the lens of algorithmic stability.
This significantly improves the existing bounds, which have at least a square root dependency. We also extend our discussion to the setting of learning with weakly dependent training examples. 

A very interesting question is to investigate whether the log dependency in the high probability analysis is an artefact of our analysis or is really essential. Another question is to extend our generalization bounds in expectation to learning with nonconvex functions.

\section*{Acknowledgments}
MK, AL and WM acknowledge support by the German Research Foundation (DFG) award KL 2698/2-1 and by the German Federal Ministry of Science and Education (BMBF) awards 01IS18051A, 031B0770E, and 01MK20014U. YL acknowledges support by NSFC under Grant No. 61806091. 

\bibliographystyle{named}

\bibliography{references}  %%% Remove comment to use the external .bib file (using bibtex).
\onecolumn
\begin{center}
  \LARGE\textbf{Supplementary Material for ``Fine-grained Generalization Analysis of Structured Output Prediction''}
\end{center}
% \newpage

%\section*{Appendix}
\appendix
\numberwithin{equation}{section}
\numberwithin{theorem}{section}
\numberwithin{figure}{section}
\numberwithin{table}{section}
\renewcommand{\thesection}{{\Alph{section}}}
\renewcommand{\thesubsection}{\Alph{section}.\arabic{subsection}}
\renewcommand{\thesubsubsection}{\Roman{section}.\arabic{subsection}.\arabic{subsubsection}}
\setcounter{secnumdepth}{-1}
\setcounter{secnumdepth}{3}

%{\red we can use cross-reference and do not need to give the theorem/lemma another label. This may cause some confusion. For example, we can just use \verb|\begin{proof}[Proof of Theorem ?]\end{proof}|}
% \begin{lemma}
% The loss function $L_\rho$ is $(\frac{2}{\rho},\ell_\infty)$-Lipschitz with respect to the scoring function $h$ for all $x \in \X$ and $y \in \Y$
%\begin{proof}[Proof of Lemma \ref{lem:lipschitzness}]
%Let $h, \tilde{h} \in \F$  be arbitrary scoring functions. Given arbitrary $(x,y ) \in \X \times \Y$,
%\begin{equation}
%    \begin{split}
%    |L_{\rho}(x,y,h)  -  L_{\rho}(x,y,\tilde{h})| \le& |(\max_{y' \ne y} L(y',y) - \frac{1}{\rho}[h(x,y) - h(x,y')]) - (\max_{y' \ne y} L(y',y) - \frac{1}{\rho}[\tilde{h}(x,y) - \tilde{h}(x,y')])|\\
%    \le& \max_{y' \ne y} |-\frac{1}{\rho}[h(x,y) - h(x,y')] + \frac{1}{\rho}[\tilde{h}(x,y) - \tilde{h}(x,y')]| \\
%    \le& \frac{1}{\rho}\left[ \max_{y'\ne y} |h(x,y') - \tilde{h}(x,y')| + |h(x,y) - \tilde{h}(x,y)|\right] \\
%    \le& \frac{1}{\rho}\left[ \max_{y \in \Y} |h(x,y) - \tilde{h}(x,y)| +  \max_{y \in \Y}|h(x,y) - \tilde{h}(x,y)|\right] \\
%    \le & \frac{2}{\rho} \max_{y \in \Y}|h(x,y) - \tilde{h}(x,y)|
%    \end{split}
%\end{equation}
%This establishes the Lipschitz continuity and finishes the proof.
% \end{proof}
% \end{lemma}

% \end{lemma}
\section{Proof of Lemma \ref{lem:lipschitzness}}

In this section, we present the proof of Lemma \ref{lem:lipschitzness}.

\begin{proof}[Proof of Lemma \ref{lem:lipschitzness}]
Let $h, \tilde{h} \in \F$  be arbitrary scoring functions. Given arbitrary $(x,y ) \in \X \times \Y$,
    \begin{align*}
    & |L_{\rho}(x,y,h)  -  L_{\rho}(x,y,\tilde{h})| \\
    & \le\Big|(\max_{y' \ne y} L(y',y) - \frac{1}{\rho}[h(x,y) - h(x,y')]) -  (\max_{y' \ne y} L(y',y) - \frac{1}{\rho}[\tilde{h}(x,y) - \tilde{h}(x,y')])\Big|\\
    &\le \max_{y' \ne y} \left|-\frac{1}{\rho}[h(x,y) - h(x,y')] + \frac{1}{\rho}[\tilde{h}(x,y) - \tilde{h}(x,y')]\right| \\
    &\le \frac{1}{\rho}\left[ \max_{y'\ne y} |h(x,y') - \tilde{h}(x,y')| + |h(x,y) - \tilde{h}(x,y)|\right] \\
    &\le \frac{1}{\rho}\left[ \max_{y \in \Y} |h(x,y) - \tilde{h}(x,y)| +  \max_{y \in \Y}|h(x,y) - \tilde{h}(x,y)|\right] \\
    &\le \frac{2}{\rho} \max_{y \in \Y}|h(x,y) - \tilde{h}(x,y)|.
    \end{align*}
This establishes the Lipschitz continuity.
 \end{proof}
\section{Proofs on Generalization Bounds with High Probability}
%\subsection{Proof Sketch}% The main idea is based on the technique recently introduced in \cite{lei2019data}.
In this section we sketch the proof of Theorem \ref{thm:main}. A key step is to control the complexity of the loss function class \eqref{eq:loss_class}, which is highly nonlinear and therefore is challenging to deal with. Our basic idea is to relate this complexity to a complexity of a \emph{linear} function class, which is easier to deal with. Indeed, define the following extended function class based on the training sample $\widetilde{S}$
\begin{equation*}
\widetilde{\F}_{p,\Lambda} := \{v \mapsto \left<w,v\right>: w \in \R^D,\|w\|_p\le \Lambda, v \in \widetilde{S}\},
\end{equation*}
where $\widetilde{S}$ is defined by
\begin{equation*}
  \widetilde{S} := \{ \Psi_f(x,y): x \in S_{|\X}, y \in \Y_f , f \in F\}.
\end{equation*}
The basic idea in constructing $\widetilde{S}$ is as follows. For each input $x$ in the original training set and each feature function $\Psi_f(.,.)$, we construct $|\Y_f|$ training examples as $\Psi_f(x,y')$, for all $y' \in \Y_f$. Therefore, the cardinality of $\widetilde{S}$ is $md|F|$. A remarkable discovery is that the complexity of the highly nonlinear loss function class $F_{p,\Lambda,\rho}$ on a dataset of size $m$ can be upper bounded by that of the much simpler linear function class $\widetilde{\F}_{p,\Lambda}$ on a dataset of size $md|F|$, via the tool of $\ell_\infty$ covering numbers.

  %We first relate the empirical $\ell_\infty$ covering numbers (see definition below) of the class \eqref{eq:loss_class} to the empirical $\ell_\infty$ covering number of the hypothesis class on an extended training set.

   % We first give the definition of $\ell_\infty$-covering numbers, which serves as our main analysis tool.
  \begin{definition}
    Let $\F$ be a class of real-valued functions defined over a space $Z$ and $S:=\{z_1, \ldots, z_m\}  \subset Z$. For any $\epsilon > 0$, the empirical $\ell_\infty$-covering number denoted by $\N(\epsilon,\F,S)$ with respect to $S$ is defined as the minimum cardinality $N$ of collection of vectors $v^1,\ldots, v^N \in \R^m$, which we refer to as a cover, such that
    \begin{equation*}
      \sup_{h \in \F}\min_{j=1,\ldots,N}\max_{i=1, \ldots, m}|h(z_i)- v_i^j| \le \epsilon.
    \end{equation*}
  \end{definition}

% Similarly define
% \begin{equation*}
% \widetilde{\F}'_{p,\Lambda} := \{v \mapsto \left<w,v\right>: w \in \R^D,\|w\|_p\le \Lambda, v \in \widetilde{S}'\},
% \end{equation*}
%  where,
% \begin{equation*}
%   \widetilde{S}' := \{ \Psi_f(x,y_f): (x,y) \in S , f \in F\}.
% \end{equation*}

% That is, the set of size $m|F|$ constructed similarly but with only the correct labels from the training set.

%   \section{Appendix}

    %   \begin{equation}
    %   F^1_{p,\Lambda} := \{(x,y) \mapsto L^1_{\rho}(x,y,h^w(x,.)): h^w \in \F_p\},
    % \end{equation}

% {\red what is the definition of $\ell_\infty$ Lipschitz continuity? Why $L_\rho$ is not Lipschitz continuous? Usually we do not need to sketch the proof in the main results section. The readers are interested in the results and the comparison with existing results. Therefore, I suggest you to re-organize this section. First present the results and then add comparison with existing work. Try to avoid introducing unnecessary definitions here. For example, it seems that the definition of covering numbers is not need to state the main results and comparisons.}

% Our results require that the loss function be $\ell_\infty$-Lipschitz.

We now formally start building the connection between the complexity of $F_{p,\Lambda,\rho}$ and that of $\widetilde{\F}_{p,\Lambda}$. %The proof is deferred to the appendix. % {\red where is the proof}
\begin{theorem}
  \label{thm:ours}
  Given the notation above, the following holds
  \begin{equation}
    \label{eq:covering_bound_1}
        \log\N(\epsilon,F_{p,\Lambda,\rho},S) \le \log \N\Big(\frac{\rho}{2|F|}\epsilon,\widetilde{\F}_{p,\Lambda},\widetilde{S}\Big).
      \end{equation}
\end{theorem}
    This result shows that one can use the covering numbers of the class $\widetilde{\F}_{\rho,\Lambda}$ to control the covering numbers of $F_{p,\Lambda\rho}$. The advantage of this result is that we now only need to control a covering number of a scalar real-valued function class as opposed to a multi-class-multi-factored class. This reduces the complexity of the problem significantly as controlling covering numbers of scalar function classes is a well-studied problem. In particular, we refer to a covering number bound of linear function classes \cite{zhang2002covering}. Note considering a large dataset only comes at a slight penalty since the covering number bound in the following lemma enjoys only a log dependency on the cardinality of dataset.
    %Further note that the larger dataset cardinality of the expanded class has only a slight penalty as bounding covering numbers of linear function classes depends on the size of training data only in $\log$ terms.
      \begin{lemma}[\cite{zhang2002covering}]
\label{thm:zhang}
        Let $\mathcal{L}$ be a class of linear functions on a set of size $n$. That is, $\mathcal{L} = \{\dotp{w}{x}, x,w \in \R^N\}$. If $\|x\|_q \le b$ and $\|w\|_p \le a$, where $2 \le q < \infty$ and $1/p + 1/q = 1$, then $\forall \epsilon > 0$,
        \begin{equation*}
          \log \N(\epsilon,\mathcal{L},n) \le 36(q -1 ) \frac{a^2b^2}{\epsilon^2} \log [2 \lceil 4ab/\epsilon + 2\rceil n + 1],
        \end{equation*}
        where $\N(\epsilon,\mathcal{L},n)$ is the worst case covering number of the class $\mathcal{L}$ on a dataset of size $n$.
      \end{lemma}
      The result controls the covering numbers by norms of the data and weights. %Note that the dependence on the cardinality of the training set is only in log terms and therefore useful to our goal.
      As a direct corollary of Lemma \ref{thm:zhang} and Theorem \ref{thm:ours}, we derive the following bound on the covering numbers of loss function classes.
      %The next immediate corollary shows how to bound the log covering numbers. The proof is simply by combining theorems  and .

      \begin{corollary}
        \label{cor:cover_bounds}
        Given the notation above, the following holds
        \begin{equation*}
          \log\N(\epsilon,F_{p,\Lambda,\rho},S) \le  C\frac{{(\Psi^*)}^2\Lambda^2|F|^2}{\epsilon^2 \rho^2} L_{\log},
        \end{equation*}
        where $C = 144(q -1)$, $L_{\log} =\log[2\lceil8\Psi^* \Lambda|F|/\epsilon \rho  +2 \rceil md|F| + 1] $.
        \end{corollary}

%\textcolor{red}{Antoine:Shouldn't the formula for $L$ actually be 
%$L=\log[2\lceil8\Psi^* \Lambda|F|/\epsilon \rho  +2 \rceil md|F| + 1] $?}
%{\color{red}: Yunwen: Here note that $L$ is also used in Lemma 5 as a function class. This may cause some confusion. We need another notation}
%{\red An important thing in writing is that we do not follow what we think logically steps by steps. We first present results and then collect the analysis in a proof. We need formal proofs instead of informal proofs}
% {\red Corollary 2 is the same as Theorem 1? Then we can remove this corollary and just say that we can use covering number bounds and entropy integral to get Theorem 1. Note there are repeated definitions of $\widetilde{L}$ and $\Psi^*$. We only need to define them once}
Now that we established bounds on the log of covering numbers, we use these bounds to give bounds on the Rademacher complexity of $F_{p,\Lambda,\rho}$. To that extent, we use Dudley's theorem to obtain a bound on the Rademacher complexities given bonds on the log covering numbers and thus establishing Theorem \ref{thm:main}.
The scoring function $h(x,y)$ can be viewed as the probability for the class $y$ given an input $x$. In our analysis we would like to treat the function $h(x,.)$ as a real-valued vector. That is the vector of class probabilities for each $y \in \Y^k$ This is possible since the set of classes $\Y^k$ is finite. Formally for any general function $f: \X \times \Y_0 \rightarrow \R$ (here, we assume a general finite output space $\Y_0$), where $\Y_0 = \{c_1, \dots, c_K\}$ is some finite set with size $K$, we denote the vector $(f(x,c_1), \ldots, f(x,c_K)) \in \R^K$ by $[f(x,\Y_0)]$. Therefore, we have $[f(x,\Y_0)]_k = f(x,c_k)$. In what follows we always use the notation $f(x,c_k)$ instead of $[f(x,\Y_0)]_k$. Note that for any $c \in \Y_0$, there is a corresponding $k_c \in [K]$ such that $c = c_{k_c}$ and therefore we have $f(x,c) = [f(x,\Y_0)]_{k_c}$ where $k_c$ is the corresponding index of $c$.

%{\red $\mathcal{Y}=\Y_1\times\cdots Y_l$ according to the first paragraph, here $\Y=\{y_1,\ldots,y_k\}$? Also $y_i$ here causes confusion with the output of the i-th example?}

% \begin{theorem}[Proof of Theorem \ref{thm:ours}]

%   Given the notation above, the following holds,
%   \begin{equation}
%     \label{eq:covering_bound_1}
%         \log\N(\epsilon,F_{p,\Lambda,\rho},S)) \le \log \N(\frac{\rho}{|F|}\epsilon,\widetilde{\F}_{p,\Lambda},\widetilde{S})
%       \end{equation}

% \end{theorem}
\begin{proof}[Proof of Theorem \ref{thm:ours}]
  We start by proving \eqref{eq:covering_bound_1}. The idea is to start with an $(\epsilon,\ell_\infty)$-cover for $\widetilde{\F}_{p,\Lambda,\rho}$ and use it to construct a $(\frac{\rho}{|F|}\epsilon,\ell_\infty)$-cover for $F_{p,\Lambda,\rho}$ with the same size. To that extent, let

  \begin{equation*}
    C := \{{\bf r}^j = (r_{1,1,1}^j, \ldots, r_{m,d,l}^j): j = [N]\} \subset \R^{mdl},
  \end{equation*}
  be an $(\epsilon,\ell_\infty)$-cover for $\widetilde{\F}_{p,\Lambda}$, where $N$ is the cardinality of the cover.
Now we use $C$ to construct a cover for $F^1_{p,\Lambda,\rho}$. To simplify notation, denote by ${\bf r}^j_{i,.,f}$ the vector $({\bf r }^j_{i,1,f}, \ldots, {\bf r}_{i,d,f}) \in \R^d$. Further let  $\Psi_f(x,.) $ denote the matrix $(\Psi_f(x,y_1)^T, \ldots, \Psi_f(x,y_d)^T) \in \R^{d \times D}$, where $y_1, \ldots, y_d$ are the elements of $\Y_f$ and  define the corresponding dot product as $\dotp{w}{\sum_{f \in F} \Psi_f(x,.)} := \left(\dotp{w}{\sum_{f \in F} \Psi_f(x,y_1)}, \ldots, \dotp{w}{\sum_{f \in F} \Psi_f(x,y_d)}\right) \in \R^d$.
  % \begin{equation}
  %   |\dotp{w}{\Psi_{f_0}(x_i,j)} - r_{i,j,f_0}| \le \epsilon.
  % \end{equation}

  Now we claim that the set
  \begin{equation*}
    \left\{ \left(L_\rho(x_1,y_1,\sum_{f \in F}{\bf r}_{1,.,f}^j),\ldots ,L_\rho(x_m,y_m\sum_{f \in F}{\bf r}_{m,.,f}^j) \right): j \in [N]\right\} \in \R^{m}
  \end{equation*}

  is an $(\frac{\rho}{|F|},\ell_{\infty})$- cover to the set

  \begin{equation*}
    \left\{ \left(L_\rho\left(x_1,y_1, \dotp{w}{\sum_{f \in F} \Psi_f(x_1,.)}\right), \ldots,L_\rho\left(x_m,y_m, \dotp{w}{\sum_{f \in F} \Psi_f(x_m,.)}\right) \right): w \in \|w\|_p \le \Lambda \right \}.
    \end{equation*}

    Indeed by the construction of $C$, we have for any $w  \in \R^D$ such that $\|w\|_p \le \Lambda$, there exists $j(w)$ such that
    \begin{equation}
      \label{eq:cover_inequality}
      \max_{i \in [m]} \max_{j \in [d]} \max_{f \in F} |{\bf r}^{j(w)}_{i,j,f} - \dotp{w}{\Psi_f(x_i,y_j)}| \le \epsilon.
    \end{equation}

    Therefore,

    \begin{equation*}
      \begin{split}
        \max_{i \in [m]} \bigg|L_\rho\bigg(x_i,y_i, \dotp{w}{\sum_{f\in F}\Psi_f(x_i,.)}\bigg) - &L_\rho(x_i,y_i,\sum_{f\in F}{\bf r}_{i,.,f}^{j(w)})\bigg|  \\
        &\le \frac{2}{\rho}\max_{i \in [m]} \left\|\sum_{f \in F} \dotp{w}{\Psi_f(x_i,.)} - \sum_{f \in F} {\bf r}^{j(w)}_{i,.,f}\right\|_\infty \\
        &= \frac{2}{\rho} \max_{i \in [m]} \max_{j \in [d]} |\sum_{f \in F} (\dotp{w}{\Psi_f(x_i,y_j)} - r_{i,j,f}^{j(w)})| \\
        & \le \frac{2|F|}{\rho} \max_{i \in [m]}\max_{j \in [d]}\max_{f \in F}|\dotp{w}{\Psi_f(x_i,y_j)} - r^{j(w)}_{i,j,f}| \\
        & \le \frac{2|F|}{\rho}\epsilon,
        \end{split}
      \end{equation*}
      where the first inequality is from $\ell_\infty$-Lipschitzness of $L_\rho$ and the second inequality is from triangule inequality. The last inequality is from the choice of $j(w)$ to satisfy \eqref{eq:cover_inequality}. Therefore we can conclude that
      \begin{equation*}
        \log\N(\epsilon,F_{p,\Lambda,\rho},S) \le \log \N(\frac{\rho}{2|F|}\epsilon,\widetilde{\F}_{p,\Lambda},\widetilde{S}).
      \end{equation*}

      % \begin{equation}
      %           \log\N(\epsilon,F^2_{p,\Lambda},S)) \le \log \N(\frac{\rho}{|F|}\epsilon,\widetilde{\F}'_{p,\Lambda},\widetilde{S}')
      %   \end{equation}
    \end{proof}

%    \begin{corollary}
%  With the above notation, the following bounds hold.
%          \begin{equation*}
%                            \Rad(F_{p,\Lambda,\rho}) \le    \frac{4}{m} + \frac{72 \sqrt{q-1}\Psi^* \Lambda |F|}{\rho\sqrt{m}}  \sqrt{\log(2md|F|[4\Psi^* \Lambda m + 3]+1)} \log m.
%                          \end{equation*}
%          \end{corollary}
% {\red It would be nice to give the Dudley's entropy integral for completeness}
Before we prove Theorem \ref{thm:main}, we first state the classical result of Dudley's entropy integral. The result is classic. Proofs can be found in \cite{bartlett2017spectrally}.
\begin{theorem}
	Let $\mathcal{F}$ be a real-valued function class taking values in $[0,1]$, and assume that $0\in \mathcal{F}$. Let $S$ be a finite sample of size $m$. We have the following relationship between the empirical Rademacher complexity  $\Rad(\mathcal{F})$ and the covering number $\N(\epsilon,\mathcal{F},S)$.
	\begin{align*}
	\Rad(\mathcal{F}) \leq \inf_{\alpha>0} \left(     4\alpha+\frac{12}{\sqrt{n}}  \int_{\alpha}^1      \sqrt{\log \N(\epsilon,\mathcal{F},S) }     \right).
	\end{align*}

\end{theorem}

                            \begin{proof}[Proof of Theorem \ref{thm:main}]
                             The proof is a straight forward application of the Dudley's entropy integral to the upper bounds obtained in Corollary \ref{cor:cover_bounds}.
            To simplify notation let $B = \frac{12 \sqrt{(q-1)}\Psi^* \Lambda |F|}{\rho}$, and $A = 16 \Psi^*\Lambda md|F|^2/\rho$, and $D=6md|F| +1$
            \begin{equation*}
              \begin{split}
                \Rad(F_{p,\Lambda,\rho}) &\le \frac{4}{m} + \frac{12}{\sqrt{m}} \int_{\frac{1}{m}}^1\sqrt{\log \N(\epsilon,F^1_{\rho,\Lambda},S)} d \epsilon\\
                & \le \frac{4}{m} + \frac{12B}{\sqrt{m}} \int_{\frac{1}{m}}^1 \frac{\sqrt{\log{(\frac{A}{\epsilon} + D)}}}{\epsilon} d \epsilon\\
                & \le \frac{4}{m} + \frac{12B}{\sqrt{m}}\sqrt{\log(Am + D)} \log{m}
                \end{split}
              \end{equation*}
              Where the first inequality is the Dudley's entropy integral. The second inequality follows from substituting the upper bounds obtained in Corollary \ref{cor:cover_bounds}. The last inequality follows by noticing that $\log{\frac{A}{\epsilon}} \le \log{Am}$ for $\epsilon \in [1/m,1]$ and therefore can be taken outside of the integral, meaning it suffices to evaluate $\int_{\frac{1}{m}}^1 \frac{1}{\epsilon} d\epsilon$.

              We then substitute the values of $B$, $A$, and $D$ to get:
%\textcolor{red}{A:will fix soon}
              \begin{equation*}
                \begin{split}
                  \Rad(F_{p,\Lambda,\rho}) &\le \frac{4}{m} + \frac{144 \sqrt{q-1}\Psi^* \Lambda |F|}{\rho\sqrt{m}}  \sqrt{\log(16\Psi^*\Lambda m^2 d |F|^2/\rho + 6md|F|+1)} \log m\\
                  & =  \frac{4}{m} + \frac{144 \sqrt{q-1}\Psi^* \Lambda |F|}{\rho\sqrt{m}}  \sqrt{\log(2md|F|[8\Psi^* \Lambda m|F|/\rho + 3]+1)} \log m.
                  \end{split}
                \end{equation*}

            \end{proof}

\section{Proofs on Generalization Bounds in Expectation}
\subsection{Regularized Risk Minimization\label{sec:proof-rrm}}
We say a function $g$ is $\lambda$-strongly convex if
\[
g(w)\geq g(w')+\langle w-w',g'(w')\rangle+\frac{\lambda}{2}\|w-w'\|_2^2,
\]
where $g'(w')$ is a subgradient of $g$ at $w=w'$.
\begin{proof}[Proof of Lemma \ref{lem:stab-reg}]
Let $S=\{z_1,\ldots,z_m\}$ and $S'=\{z_1',\ldots,z_m'\}$ be two datasets that differ only by a single example. Without loss of generality, we assume $z_i=z_i'$ for $i\in[m-1]$. Since $R_S$ is convex (note $L_\rho$ is convex since we consider linear models here and a maximum of convex functions is again convex) we know that $R_S^\lambda$ is $\lambda$-strongly convex. It then follows that
\[
R^\lambda_{S'}(w_S)-R^\lambda_{S'}(w_{S'})\geq\frac{\lambda}{2}\|w_S-w_{S'}\|_2^2.
\]
Furthermore, it follows from the definition of $w_S$ and the $(2/\rho,\ell_\infty)$-Lipschitz continuity of $L_\rho$ that
\begin{align*}
  R^\lambda_{S'}(w_S)-R^\lambda_{S'}(w_{S'})
   &=R^\lambda_{S}(w_S)-R^\lambda_{S}(w_{S'}) +\frac{L_\rho(x_n',y_n',h^{w_S})-L_\rho(x_n,y_n,h^{w_S})}{m}
   + \frac{L_\rho(x_n,y_n,h^{w_{S'}})-L_\rho(x_n',y_n',h^{w_{S'}})}{m} \\
  & \leq \frac{2}{m\rho}\max_{y\in\mathcal{Y}}\Big(|h^{w_S}(x_n,y)-h^{w_{S'}}(x_n,y)| +|h^{w_S}(x'_n,y)-h^{w_{S'}}(x'_n,y)|\Big)\\
  & =\frac{2}{m\rho}\max_{y\in\mathcal{Y}}\Big(|\langle w_S-w_{S'},\Psi(x_n,y)\rangle|+|\langle w_S-w_{S'},\Psi(x_n',y)\rangle|\Big)\\
  & \leq \frac{2}{m\rho}\max_{y\in\mathcal{Y}}\|w_S-w_{S'}\|_2\|\Psi(x_n,y)+\Psi(x_n',y)\|_2\Big)\leq \frac{4\kappa}{m\rho}\|w_S-w_{S'}\|_2,
\end{align*}
where we have used the assumption $\max_{x,y}\|\Psi(x,y)\|_2\leq\kappa$. We can combine the above two inequalities to show
\[
\|w_S-w_{S'}\|_2\leq \frac{8\kappa}{m\rho\lambda}.
\]
By using the $(2/\rho,\ell_\infty)$-Lipschitz continuity again, we get
\begin{align*}
\sup_{x,y}\big|L_\rho(x,y,h^{w_S})-L_\rho(x,y,h^{w_S'})\big| 
&\leq \frac{2}{\rho}\sup_{x,y}\big|h^{w_S}(x,y)-h^{w_{S'}}(x,y)\big|\\
&= \frac{2}{\rho}\sup_{x,y}\big|\langle w_S-w_{S'},\Psi(x,y)\rangle\big|
\leq \frac{2\kappa}{\rho}\|w_S-w_{S'}\|_2\leq \frac{16\kappa^2}{m\rho^2\lambda}.
\end{align*}
\end{proof}
\begin{proof}[Proof of Theorem \ref{thm:reg}]
  According to Lemma \ref{lem:gen-stab} and Lemma \ref{lem:stab-reg}, we know
  \begin{equation}\label{reg-1}
  \mathbb{E}\big[R^\lambda(w_S)-R^\lambda_S(w_S)\big]=\mathbb{E}\big[R(w_S)-R_S(w_S)\big]\leq \frac{16\kappa^2}{m\rho^2\lambda}.
  \end{equation}
  By the definition of $w_S$ we know $R_S^\lambda(w_S)\leq R_S^\lambda(w^\lambda)$. Note $w^\lambda$ is independent of the sample, and we know $\mathbb{E}\big[R_S^\lambda(w^\lambda)\big]=R^\lambda(w^\lambda)$.
  It then follows that
  \begin{align*}
    \mathbb{E}\big[R^\lambda(w_S)-R^\lambda(w^\lambda)\big]  & = \mathbb{E}\big[R^\lambda(w_S)-R_S^\lambda(w_S)\big]  
     + \mathbb{E}\big[R_S^\lambda(w_S)-R_S^\lambda(w^\lambda)\big] + \mathbb{E}\big[R_S^\lambda(w^\lambda)-R^\lambda(w^\lambda)\big]\\
    & \leq \mathbb{E}\big[R^\lambda(w_S)-R_S^\lambda(w_S)\big].
  \end{align*}
  We can combine the above inequality and \eqref{reg-1} together and get \eqref{reg-a}.
  This proves the first inequality.
  By the definition of $w^\lambda$ we know $R^\lambda(w^\lambda)\leq R^\lambda(w^*)$. We can plug this inequality back into \eqref{reg-a} and get
 \[
  \mathbb{E}\big[R^\lambda(w_S)\big]-R(w^*)\leq \frac{\lambda}{2}\|w^*\|_2^2+ \frac{16\kappa^2}{m\rho^2\lambda}.
  \]

  If we choose $\lambda=\frac{4\sqrt{2}\kappa}{\sqrt{m}\rho \|w^*\|_2}$, then we get
  \[
  \mathbb{E}\big[R^\lambda(w_S)\big]-R(w^*)\leq \frac{4\sqrt{2}\kappa\|w^*\|_2}{\sqrt{m}\rho}.
  \]
  The proof is complete.
\end{proof}

\subsection{Stochastic Gradient Descent\label{sec:proof-sgd}}
\begin{proof}[Proof of Lemma \ref{lem:stab-sgd}]
According to Lemma \ref{lem:lipschitzness}, for any $w,w'$ there holds
\begin{align*}
|L_{\rho}(x,y,h^{w})  -  L_{\rho}(x,y,h^{w'})| &\leq \frac{2}{\rho}\max_{y \in \Y}|\langle w-w',\Phi(x,y)\rangle|\leq \frac{2\kappa}{\rho}\|w-w'\|_2.
\end{align*}
It then follows that
\begin{equation}\label{gradient-bound}
  \|L'_{\rho}(x,y,h^{w})\|_2\leq \frac{2\kappa}{\rho}.
\end{equation}
Without loss of generality, we assume $z_i=z_i'$ for $i\in[m-1]$.
We now analyse how $\|w^{(t)}-\hat{w}^{(t)}\|_2$ changes along the iterations. Consider two cases at the $t$-th iteration. If $i_t\neq m$, then
\begin{align*}
  &\|w^{(t+1)}-\hat{w}^{(t+1)}\|_2^2 = \big\|w^{(t)}-\eta_tL_\rho'(x_{i_t},y_{i_t},h^{w^{(t)}})-\hat{w}^{(t)}+\eta_tL_\rho'(x_{i_t},y_{i_t},h^{\hat{w}^{(t)}})\big\|\\
  &=\|w^{(t)}-\hat{w}^{(t)}\|_2^2 
  +\eta_t^2\|L_\rho'(x_{i_t},y_{i_t},h^{w^{(t)}})-L_\rho'(x_{i_t},y_{i_t},h^{\hat{w}^{(t)}})\|_2^2
  -2\eta_t\langle w^{(t)}-\hat{w}^{(t)}, L_\rho'(x_{i_t},y_{i_t},h^{w^{(t)}}) - L_\rho'(x_{i_t},y_{i_t},h^{\hat{w}^{(t)}})\rangle.
\end{align*}
The convexity of $L_\rho$ implies that
\[
\langle w^{(t)}-\hat{w}^{(t)}, L_\rho'(x_{i_t},y_{i_t},h^{w^{(t)}}) - L_\rho'(x_{i_t},y_{i_t},h^{\hat{w}^{(t)}})\rangle\geq 0.
\]
Together with~\eqref{gradient-bound}, this implies
\begin{equation}\label{stab-sgd-1}
 \|w^{(t+1)}-\hat{w}^{(t+1)}\|_2^2\leq \|w^{(t)}-\hat{w}^{(t)}\|_2^2+16\eta_t^2\kappa^2\rho^{-2}.
\end{equation}
If $i_t=m$, then
\[
\|w^{(t+1)}-\hat{w}^{(t+1)}\|_2\leq \|w^{(t)}-\hat{w}^{(t)}\|_2+\eta_t\|L_\rho'(x_{m},y_{m},h^{w^{(t)}})-L_\rho'(x'_{m},y'_{m},h^{\hat{w}^{(t)}})\|_2.
\]

It then follows from~\ref{gradient-bound} and the elementary inequality $(a+b)^2\leq (1+r)a^2+(1+1/r)b^2$ that
\[
\|w^{(t+1)}-\hat{w}^{(t+1)}\|_2^2\leq (1+r)\|w^{(t)}-\hat{w}^{(t)}\|_2^2+(1+1/r)8\eta_t^2\kappa^2\rho^{-2}.\]

Since with probability $1-1/m$ we have $i_t\neq m$ and with probability $1/m$ we have $i_t=m$, we can combine the above two cases to show
\[
\mathbb{E}_{i_t}\big[\|w^{(t+1)}-\hat{w}^{(t+1)}\|_2^2\big]\leq (1+r/m)\|w^{(t)}-\hat{w}^{(t)}\|_2^2+ (1+1/(rm))8\eta_t^2\kappa^2\rho^{-2}.
\]
We can apply the above inequality recursively and get
\[
\mathbb{E}_{A}\big[\|w^{(t+1)}-\hat{w}^{(t+1)}\|_2^2\big]\leq 16(1+1/(rm))\kappa^2\rho^{-2}\sum_{j=1}^{t}\eta_j^2(1+r/m)^{t-j}
\]
If we choose $r=m/t$ we know that for any $j\leq t$,
\[
(1+r/m)^{t-j}\leq (1+1/t)^t\leq e.
\]
and therefore
\[
\mathbb{E}_{A}\big[\|w^{(t+1)}-\hat{w}^{(t+1)}\|_2^2\big]\leq 16e(1+t/m^2)\kappa^2\rho^{-2}\sum_{j=1}^{t}\eta_j^2.
\]
The proof is complete.
\end{proof}
\begin{proof}[Proof of Theorem \ref{thm:gen-sgd}]
  According to Lemma \ref{lem:stab-sgd} we know
  \[
  \mathbb{E}_{A}\big[\|w^{(t+1)}-\hat{w}^{(t+1)}\|_2\big]\leq \sqrt{16e}(1+\sqrt{t}/m)\kappa\rho^{-1}\sqrt{t}\eta.
  \]
  It then follows from the convexity of the $\ell_2$ norm that
  \[
  \mathbb{E}_{A}\big[\|\bar{w}^{(T)}-\bar{\hat{w}}^{(T)}\|_2\big]\leq 7(1+\sqrt{T}/m)\kappa\rho^{-1}\sqrt{T}\eta.
  \]
  This together with Lemma \ref{lem:lipschitzness} shows the following uniform stability bounds
  \[
  \sup_{z}\mathbb{E}_{A}\big[\big|L_\rho(x,y,h^{\bar{w}^{(T)}})-L_\rho(x,y,h^{\bar{\hat{w}}^{(T)}})\big|\big]\leq 14(1+\sqrt{T}/m)\kappa^2\rho^{-2}\sqrt{T}\eta.
  \]
  It then follows from Lemma \ref{lem:gen-stab} that
  \[
  \mathbb{E}\big[R(\bar{w}^{(T)})-R_S(\bar{w}^{(T)})\big]\leq O\big(\kappa^2(\sqrt{T}+T/m)\eta\big).
  \]
  Furthermore, standard convergence rate analysis of SGD shows that~\cite{bottou2018optimization}
  \[
  \mathbb{E}_A\big[R_S(\bar{w}^{(T)})-R_S(w^*)\big]\leq O\Big(\frac{1+T\kappa^2\eta^2}{T\eta}\Big).
  \]
  It then follows that
  \begin{align*}
    \mathbb{E}\big[R(\bar{w}^{(T)})-R(w^*)\big]&=\mathbb{E}\big[R(\bar{w}^{(T)})-R_S(\bar{w}^{(T)})\big]
    + \mathbb{E}\big[R_S(\bar{w}^{(T)})-R_S(w^*)\big] + \mathbb{E}\big[R_S(w^*)-R(w^*)\big]\\
    & \leq O\big((\sqrt{T}+T/m)\eta\kappa^2\big)+O\Big(\frac{1+T\kappa^2\eta^2}{T\eta}\Big).
  \end{align*} 
  The proof is complete.
\end{proof}
\section{Proofs on Learning with Weakly Dependent Examples}
In this section we prove our results on Weakly Dependent Examples. We follow the standard analysis that was introduced in \cite{yu1994rates} and followed by \cite{mohri2008rademacher,meir2000nonparametric}. Recall that the goal is to bound the following probabilit:  
  \begin{equation}
    \label{eq:uniform_conv}
    \Pro \left\{ \sup_{h\in \F}\left(  R(h) - \hat{R}_S(h) \right) > \epsilon \right\}.
  \end{equation}
  Standard uniform convergence results can not be employed here since $z_i^k$ and $z_i^{k'}$ are not independent for $k \ne k'$ and  all $i \in [m]$. Uniform convergence results for mixing sequence have  been studied in the literature, e.g., \cite{meir2000nonparametric,mohri2008rademacher}). Bounds in the aforementioned works are proved via the so called \textit{independent blocks} technique. The main difference here is that in our settings we have a hierarchical dependence. At the high level, we have independence between documents, while at the low level, we have weak dependence between the sequence of examples within each of the documents. This is unlike previous work which considered only learning with one level of weakly dependent data. Our work can also be extended to the case when we have a higher hierarchy where different levels are weakly dependent with different levels of weak dependence.

 Most of the following is based on the standard analysis of \textit{independent blocks} technique for mixing sequences that was introduced in \cite{yu1994rates} and later employed by \cite{mohri2008rademacher,meir2000nonparametric}.

For simplicity, assume that $J=2\mu a$ for positive integers $\mu$ and $a$. Let $H_j = \{i:2(j-1)a + 1 \le i \le (2j-1)a\}$ and $T_j = \{i:(2j-1)a + 1 \le i \le 2ja\}$. Further let $Z_i^{H_j} =  (z^k_i)_{k \in H_j}$ and similarly $Z_i^{T_j} = (z^k_i)_{k \in T_j}$, for $j \in [\mu]$. We then define the \textit{even} dataset $S^0$

\begin{equation}
  S^0 = \{(Z_i^{H_j})_{j \in [\mu]}\}_{i = 1}^m,
  \end{equation}
and the \textit{odd} dataset $S^1$
\begin{equation}
  S^1 = \{(Z_i^{T_j})_{j \in [\mu]}\}_{i = 1}^m.
  \end{equation}
  That is we divide each document into $2\mu$ blocks each containing $a$-consecutive elements. The even dataset $S^0$ is formed as follows: for each document in the data set $S$, include only blocks that have even index while the odd dataset is formed similarly. Therefore, any document in $S^0$ has blocks that are sampled $a$ time steps apart.

  Let $L^{a}_h(Z^{H_j}) = \frac{1}{a}\sum_{k \in H_j} L_h(z^k)$ and similarly $L^{a}_h(Z^{T_j}) = \frac{1}{a} \sum_{k \in T_j} L_h(z^k)$. The first step is to bound the probability \eqref{eq:uniform_conv} by another probability that depends only on the even dataset $S^0$. In this way we can work with sequences of blocks that are $a$ steps apart from each other and therefore the dependence between those blocks becomes weaker. The next lemma is a standard lemma \cite{yu1994rates}. It relates the probability of estimation error on the full set of sequences to that on the even set of sequences. 
  \begin{lemma}
  With the notaion above, we have 
  \begin{equation}
    \label{eq:even_prob}
          \Pro \left\{ \sup_{h\in \F}\left(R(h) - \hat{R}_S(h)\right) > \epsilon \right\} \le 2\Pro \left\{ \sup_{h\in \F}\left(  R(h) -\hat{R}_{S^0}(h)\right)   > \epsilon \right\}.
    \end{equation}
  \end{lemma}
  \begin{proof}
  According to basic properties of probabilities, we have
  \begin{equation*}
  \begin{split}
      &\Pro \left\{ \sup_{h\in \F}\left(R(h) - \hat{R}_S(h)\right) > \epsilon \right\} \\
      =&\Pro \left\{ \sup_{h\in \F}\left(  \frac{R(h) -\hat{R}_{S^0}(h)}{2}  + \frac{R(h) - \hat{R}_{S^1}(h)  }{2}  \right)  > \epsilon \right\}\\
      \le& \Pro \left\{ \sup_{h\in \F}\left(  R(h) -\hat{R}_{S^0}(h)\right)  + \sup_{h\in \F}\left(R(h) - \hat{R}_{S^1}(h)    \right)  > 2\epsilon \right\} \\
      \le& \Pro \left\{ \sup_{h\in \F}\left(  R(h) -\hat{R}_{S^0}(h)\right)   > \epsilon \right\} +\Pro \left\{ \sup_{h\in \F}\left(R(h) - \hat{R}_{S^1}(h)    \right)  > \epsilon \right\} \\
      =& 2\Pro \left\{ \sup_{h\in \F}\left(  R(h) -\hat{R}_{S^0}(h)\right)   > \epsilon \right\}.
      \end{split}
    \end{equation*}
    Here the first inequality is due to the convexity of the $\sup$, the second is combination of the union bound and the fact that for the sum to exceed $2\epsilon$, at least one of the summands has to exceed $\epsilon$, and the last inequality is due to stationarity of the process generating the documents.
\end{proof}
    Now that we have reduced the problem to the study of the even blocks, the remaining problem is that those blocks are still weakly dependent. Thus, the next step is to reduce the problem further to one with independent blocks so that we can apply standard techniques for independent data. The main idea is to replace the blocks in $S^0$ with another set of blocks each of which has the same marginal distribution but are independent of each other. Specifically, for $j \in [\mu]$ and  $i \in [m]$, we introduce the random variable $\tilde{Z}_i^{H_j} = (\tilde{z}^k_i)_{i \in m, k \in H_j}$ to have the same distribution as $Z_i^{H_j}$ and such that the set of random variables $\{\tilde{Z}_i^{H_j}\}_{i \in [m],j \in [\mu]}$ are independent and similarly we construct $\{\tilde{Z}^{T_j}_i\}_{i \in [m], j\in[\mu]}$. In what follows we use $\widetilde{S}$, $\widetilde{S}^0$, $\widetilde{S}^1$ to denote the datasets with independent blocks. The goal then is to establish a connection between \eqref{eq:even_prob} and a similar quantity that depends only on the independent blocks.

    A key result introduced in \cite{yu1994rates} relates the means of the original data to that of the independent counterpart. The original result was introduced in the case of only one sequence. We present here  a slightly modified lemma that accounts for a set of independent documents. The proof is provided at the end of this section. %See the appendix for the proof. %{\red the proof?}

    \begin{lemma}
      \label{lem:yu}
      Let $f$ be a measurable function of the set of even blocks $S^0$ that is bounded by $M$, then we have,
      \begin{equation}
        |\E_{S^0}[h] - \E_{\tilde{S^0}}[h] |  \le m(\mu -1 )M \beta(a)
      \end{equation}

      % \begin{proof}
      %   To be written later. {\color{red} maybe also try and find a better bound}.
      %   \end{proof}
      \end{lemma}

     The next direct corollary controls the probability on the weakly dependent blocks to that of independent blocks. 
     \begin{corollary}
     With the notation above, we have
     \begin{equation}
       \label{eq:prob_indep}
       \begin{split}
         & \Pro \left\{ \sup_{h\in \F}\left(  R(h) -\hat{R}_{S^0}(h)\right)   > \epsilon \right\} 
         \le \tilde{\Pro} \left\{ \sup_{h\in \F}\left(  R(h) -\hat{R}_{\widetilde{S}^0}(h)\right)   > \epsilon \right\}  + m(\mu -1)\beta(a),
         \end{split}
     \end{equation}
     where $\tilde{\Pro}$ denote the probability measure on independent data. 
     \end{corollary}
     \begin{proof}
         We can apply Lemma \ref{lem:yu} with $f$ being $\I\left\{ \sup_{h \in \F}\left( R(h) - \hat{R}_{S^0} \right) > \epsilon \right\}$ to get the inequality. 
     \end{proof}
     
     We note also that $\E_{\widetilde{S}^0}[\hat{R}_{\widetilde{S}^0}] = R(h)$ due to the linearity of expectation and the fact that the blocks in $\widetilde{S}^0$ have the same distributions as blocks of $S^0$.  We further note that $\hat{R}_{\widetilde{S}^0}(h) = \frac{1}{m \mu} \sum_{i \in [m],j \in [\mu]}L^{a}_h(\tilde{Z}^{H_j}_i)$. The first term in the right hand side of \eqref{eq:prob_indep} can be bounded via standard uniform convergence techniques, because we have the probability over independent blocks, by considering the class $(L^a_h \circ \F)  = \{\frac{1}{a}\sum_{k = 1}^a L_h(.): h\in \F\}$ (see for example \cite{mohri2018foundations,mohri2008rademacher}). Therefore, we can apply standard analysis to get for any $\epsilon' > 0$,

     \[
       \tilde{\Pro} \left\{ \sup_{h\in \F}\left(  R(h) -\hat{R}_{\widetilde{S}^0}(h)\right) - 2\E_{\widetilde{S}^0}[\Rad_{\tilde{S^0}}((L^a_h \circ \F))]   > \epsilon' \right\} \\ \le 2\exp\left(\frac{-2\mu m \epsilon'^2}{M^2}\right).
     \]

       For the left hand side of the last inequality to match the right hand side of \eqref{eq:prob_indep}, we set $\epsilon' = \epsilon - 2\E_{\widetilde{S}^0}[\Rad_{\tilde{S^0}}((L^a_h \circ \F))]$, thus,

       \begin{equation}
         \label{eq:mcdiarmids}
                \tilde{\Pro} \left\{ \sup_{h\in \F}\left(  R(h) -\hat{R}_{\widetilde{S}^0}(h)\right) > \epsilon \right\} \le 2\exp\left(\frac{-2\mu m \epsilon'^2}{M^2}\right)
\end{equation}

\begin{theorem}
Let $\delta > 2m(\mu-1)\beta(a)$ with probability at least $1- \delta$, for all $h \in \F$, we have
  \begin{equation}
    \label{eq:Rademacher_mixing}
    \begin{split}
    R(h) \le \hat{R}_S(h) + 2\E_{\widetilde{S}^0}[\Rad_{\tilde{S^0}}((L^a_h \circ \F))] + \frac{M\sqrt{\log\left(\frac{2}{\delta- 2m(\mu-1)\beta(a)}\right)}}{\sqrt{2\mu m}}.
    \end{split}
  \end{equation}
\end{theorem}
\begin{proof}
We combine inequalities \eqref{eq:mcdiarmids}, \eqref{eq:prob_indep}, and \eqref{eq:even_prob}, and get
\begin{equation}
  \begin{split}
    &\Pro \left\{ \sup_{h\in \F}\left(R(h) - \hat{R}_S(h)\right) > \epsilon \right\}
     \le 2 \exp \left( \frac{-2\mu m \epsilon'^2}{M^2} \right) + 2m(\mu -1) \beta(a).
    \end{split}
  \end{equation}
  Therefore, for $\delta > 2m(\mu-1)\beta(a)$ with probability at least $1- \delta$, for all $h \in \F$, we get
  \begin{equation}
    % \label{eq:Rademacher_mixing}
    \begin{split}
    R(h) \le \hat{R}_S(h) + 2\E_{\widetilde{S}^0}[\Rad_{\tilde{S^0}}((L^a_h \circ \F))] + \frac{M\sqrt{\log\left(\frac{2}{\delta- 2m(\mu-1)\beta(a)}\right)}}{\sqrt{2\mu m}}.
    \end{split}
  \end{equation}
\end{proof}

  We now aim to control the second term in \eqref{eq:Rademacher_mixing} by the worst-case $\ell_\infty$-covering number of the dataset. To that end, we first show that the covering number of the class $(L_h^a \circ \F)$ on the dataset $\widetilde{S}^0$ is indeed upper bounded by the covering number of the class $(L_h \circ \F)$ on the set $\widetilde{S}$. This is summerized in the following lemma. 
  
  \begin{lemma}
  \label{lem:cov_weak}
  Given the notation above the following inequality holds
  \begin{equation*}
         \N(\epsilon, (L^a_h \circ \F),\widetilde{S}^0) \le  \N(\epsilon,(L_h \circ \F), \widetilde{S} ).
       \end{equation*}
\end{lemma}
\begin{proof}
  First we show that the covering number defined on $\widetilde{S}$ indeed upper bounds the covering number defined on  $\widetilde{S}^0$ . To see this, let $C:=\{{\bf r}^{j} = (r^j_{1,1}, \ldots, r^j_{m,l}): j \in [N]\} \in \R^{ml}$ be an $\ell_\infty$ cover for the class $(L_h \circ \F)$ on the dataset $\widetilde{S}$, therefore, for each $h \in \F$, there is a $j^{(h)} \in [N]$ such that,
  \begin{equation*}
    \max_{k \in [l], i \in [m]}|L_h(\tilde{z}_i^k) - r^{j(h)}_{i,k}| \le \epsilon.
    \end{equation*}

    Hence,
    \begin{equation*}
      \begin{split}
        \max_{i \in m k \in \mu} |L_h^a(\tilde{Z}^{H_j}_i) - \frac{1}{a} \sum_{t \in H_j} r_{i,t}^{j(h)}| 
        & = \max_{i \in m, k \in \mu} |\frac{1}{a}\sum_{t \in H_j}L_h(\tilde{z}^{t}_i) - \frac{1}{a} \sum_{t \in H_j} r_{i,t}^{j(h)}| \\
        & \le \max_{i \in m, k \in \mu, t \in H_j } |L_h(\tilde{z}^t_i) - r^{j(h)}_{i,t}| \\
        & \le \max_{i \in m, t \in [J] } |L_h(\tilde{z}^t_i) - r^{j(h)}_{i,t}| = \epsilon.
        \end{split}
      \end{equation*}
      It then follows that
      \begin{equation*}
         \N(\epsilon, (L^a_h \circ \F),\widetilde{S}^0) \le  \N(\epsilon,(L_h \circ \F), \widetilde{S} ).
       \end{equation*}

\end{proof}
  
  %Finally, we note that the distributions of $\widetilde{S}$ and $S$ are defined on the same space. Therefore, $\sup_{\widetilde{S}} \N(\epsilon, (L_h \circ \F) ,\widetilde{S}) = \sup_{S} \N(\epsilon, (L_h \circ \F) ,S)$ since the worst case covering number depends only on the metric structure of the psudo-metric space rather than the probability measure defined on it. %{\color{red} where is the proof?}

% The next lemma relates the covering number on the even dataset to that of the full dataset. 

       Furthermore, let $\N(\epsilon,(L_h \circ \F),mJ) = \sup_{\widetilde{S} \in \Z^{mJ}}\N(\epsilon, (L_h \circ \F),\widetilde{S})$ be the worst case covering number on a  dataset of size $mJ$. %Note that the covering number is distribution independent quantity. 
       
       Now we are ready to give the proof of Theorem \ref{thm:weakly}.
       
\begin{proof}[Proof of Theorem \ref{thm:weakly}] 
       By Dudley's entropy integral we have

       \begin{equation*}
       \begin{split}
         \Rad_{\widetilde{S}^0}((L_h \circ \F)) &\le \frac{4}{m\mu} + \frac{12}{\sqrt{m\mu}} \int_{\epsilon}^1 \sqrt{\log \N(\epsilon, (L^a_h \circ \F),\widetilde{S}^0)} d\epsilon\\
         &\le \frac{4}{m\mu} + \frac{12}{\sqrt{m\mu}} \int_{\epsilon}^1 \sqrt{\log \N(\epsilon,(L_h \circ \F), \widetilde{S} )} d\epsilon\\
         &\le \frac{4}{m\mu} + \frac{12}{\sqrt{m\mu}} \int_{\epsilon}^1 \sqrt{\log \N(\epsilon,(L_h \circ \F), mJ )} d\epsilon,
         \end{split}
       \end{equation*}
       where the second inequality is from lemma \ref{lem:cov_weak} and the third inequality is by definition of worst case covering number. 

       Now substituting in \eqref{eq:Rademacher_mixing}, we get the following bound, for $\delta > 2m(\mu-1)\beta(a)$ with probability at least $1- \delta$, for all $h \in \F$
  \begin{equation}
    % \label{eq:Rademacher_mixing}
    R(h) \le \hat{R}_S(h) +  \frac{8}{m\mu} + \frac{24}{\sqrt{m\mu}} \int_{\epsilon}^1 \sqrt{\log(\N(\epsilon, (L_h \circ \F),mJ))} d\epsilon + \frac{M\sqrt{\log\left(\frac{2}{\delta- 2m(\mu-1)\beta(a)}\right)}}{\sqrt{2\mu m}}.
  \end{equation}

  Recall that $\mu = \frac{J}{2a}$, the bound can be written as
    \begin{equation}
    % \label{eq:Rademacher_mixing}
    R(h) \le \hat{R}_S(h) +  \frac{16a}{mJ} + \frac{24\sqrt{2a}}{\sqrt{mJ}} \int_{\epsilon}^1 \sqrt{\log(\N(\epsilon, (L_h \circ \F),mJ))} d\epsilon + \frac{M\sqrt{a}\sqrt{\log\left(\frac{2}{\delta- 2m(\frac{J}{2a}-1)\beta(a)}\right)}}{\sqrt{J m}}.
  \end{equation}

  Now we would like to apply this theory to models that decompose as a factor graphs as above. Thus, for the loss class $F_{p,\Lambda,\rho}$, we have the following bound:
  \[
      R(h) \le \hat{R}_S(h) +  \frac{32a}{mJ}   + \frac{288 \sqrt{2a(q-1)}\Psi^* \Lambda |F|}{\rho\sqrt{mJ}}  \sqrt{\log(2mJd|F|[8\Psi^* \Lambda mJ/\rho + 3]+1)} \log mJ
      + \frac{M\sqrt{a}\sqrt{\log\left(\frac{2}{\delta- 2m(\frac{J}{2a}-1)\beta(a)}\right)}}{\sqrt{J m}}.
    \]
    \end{proof}
    
    Now we give the proof of lemma \ref{lem:yu}. Before we prove it, we begin by defining some basic quantities and stating a useful basic lemma. Let $Q$ be a measure defined on the a product space $(\Omega_1 \times \Omega_2, \Sigma_1 \times \Sigma_2)$, where $(\Omega_k,\Sigma_k)$ for $k \in [2] $ are two measurable spaces. Let $Q_k$ be the marginal distribution of $Q$ on $(\Omega_k,\Sigma_k)$. We define the following quantity:
\begin{equation*}
    \beta(\Sigma_1, \Sigma_2, Q) = \E \sup_{B \in \Sigma_2}|Q(B|\Sigma_1) - Q_2(B)|.
\end{equation*}
The following lemma provides the basic building block for our proof. It controls the difference of expectation with respect to $Q$ and $Q_1 \times Q_2$  of bounded functions by their $\beta$-coefficient.  
  \begin{lemma}{\cite{yu1994rates}}
  \label{lem:yubasic}
  Let $h: \Omega_1 \times \Omega_2 \rightarrow \R$ be bounded by $M$ and measurable. Let $P$ be the product measure $Q_1 \times Q_2 $. Then, the following holds
  \begin{equation*}
      |E_Q[h] - E_p[h]| \le M \beta(\Sigma_1,\Sigma_2, Q )
  \end{equation*}
  \end{lemma}
 The following is a modification of Corollary 2.7 in \cite{yu1994rates}. Here, we give a fine grained analysis, that is useful if the sequence of random variables has different mixing coefficients that change over time.
  \begin{corollary}
  Let $m \ge 1$ and $h: \Pi_{i =1}^m \Omega_i \rightarrow \R $ be a $(\Pi_{i =1 }^m\Omega_i,\Pi_{i =1 }^m\Sigma_i )$-measurable function that is bounded by $M$. Denote by $Q$ a probability measure defined on the product space $(\Pi_{i =1 }^m\Omega_i,\Pi_{i =1 }^m\Sigma_i )$. Further let $Q_i$ be the marginal probability measure of $Q$ defined on $(\Omega_i,\Sigma_i)$ and $Q^{i}$ be the marginal probability measure of $Q$ on $(\Pi_{j =1 }^i \Omega_j, \Pi_{j = 1}^i \Sigma_j)$, define

  \begin{equation*}
      \beta_i(Q) = \beta(\Pi_{j =1}^i \Sigma_j,\Sigma_{i+1}, Q^{i+1}).
  \end{equation*}
  Let $P = \Pi_{i=1}^m Q_i$. Then
  \begin{equation*}
      |\E_{P}[h] - \E_{Q}[h]| \le M \sum_{i =1}^{m-1} \beta_i(Q)
  \end{equation*}

  \begin{proof}
  The proof is by induction.
  \begin{itemize}
      \item The base case $m =2$ is by Lemma \ref{lem:yubasic}.
      \item Step: assume that the statement holds for $m-1$, and let $\hat{P} = \Pi_{i=1}^{m-1} Q_i$. Then
      \begin{equation*}
      \begin{split}
         |\E_{P}[h] - \E_{Q}[h]| &= |\E_{Q_m}\E_{\hat{P}}[h] - \E_{Q_m}\E_{Q^{m-1}}[h] + \E_{Q_m}\E_{Q^{m-1}}[h]  - \E_{Q}[h]| \\
         & \le |\E_{Q_m}\E_{\hat{P}}[h] - \E_{Q_m}\E_{Q^{m-1}}[h]| + |\E_{Q_m}\E_{Q^{m-1}}[h]  - \E_{Q}[h]| \\
         &\le \E_{Q_m}|\E_{\hat{P}}[h] - \E_{Q^{m-1}}[h]| + |\E_{Q_m}\E_{Q^{m-1}}[h]  - \E_{Q}[h]| \\
         & \le M \sum_{i =1}^{m-2}\beta_i(Q) + M \beta_{m-1} = M \sum_{i =1}^{m-1}\beta_i(Q),
         \end{split}
      \end{equation*}
      where in the first equality we introduce the expectation with respect to $Q_m \times Q^{m-1}$, the first inequality is by the triangule inequality, the second inequality is by Jensen's inequality while in the last inequality, the first term follows from the induction hypothesis and the second is by Lemma \ref{lem:yubasic}.
  \end{itemize}
  \end{proof}

  \end{corollary}
  Note that if set $\beta(Q)  = \max_{i \in [m-1]} \beta_i(Q)$, we get the same result as in corollary 2.7 in \cite{yu1994rates}. Our corollary can be useful when $\beta_i$ is different for each $i$.
  
Now we are ready to prove lemma \ref{lem:yu}.
\begin{proof}[Proof of Lemma \ref{lem:yu}]

The statement directly follows from the last corollary. Recall that $S^0$ is a set of $m$ sequence of blocks each with length $a$, that it can be arranged as the sequence, $S^0 = (Z_1^{H_1}, \ldots, Z_m^{H_j})$. Let $\Omega_i = \Z^a$, and $\Sigma_i = \sigma(Z_k^{H_l})$ for $i \in [\mu m]$, where $\sigma(Z)$ denotes the sigma algebra generated by the random variable $Z$ and $k = \lfloor i / \mu \rfloor$ and $l = (i \mod \mu)$. Now we note that $\beta_i(Q) = 0$, whenever $i$ is divisible by $\mu$ and $\beta_i(Q) = \beta(a)$ otherwise. Hence the following holds
\begin{equation}
        |\E_{S^0}[h] - \E_{\tilde{S^0}}[h] |  \le m(\mu -1 )M \beta(a).
      \end{equation}
\end{proof}

\section{Features Extracted From Neural Networks}

In this section, we sketch how one can combine the strategies presented in this paper with other bounds to obtain generalization guarantees for structured output prediction when the features are obtained via Deep Neural Networks. 

Let $\mathcal{F}$ be a class of functions from $\mathcal{X}$ to a space $Z$ endowed with the norm $\|.\|_z$ and let $\mathcal{H}$ be the class of linear functions $\theta$ from $Z$ to $\mathbb{R}$  such that $\|\theta\|_{*}\leq \Lambda$ where $\|.\|_{*}$ denotes the dual norm to $\|.\|_z$.  For a sample set $S=\{x_1,\ldots,x_m\}\subset \mathcal{X}$  denote by $\mathcal{N}_{z,\infty}(\epsilon,\mathcal{F},S)$ be the covering number of $\mathcal{F}$ with $\ell_\infty$ and $\|.\|_z$ norms, i.e., the smallest $N$ such that we have cover $\{f^1,\ldots,f^N\} \subset \mathcal{F}$ such that  $\forall f\in \mathcal{F}$ there exists $j\leq N$ such that for all $i\leq n$: $$\|f(x_i)-f^j(x_i)\|_z\leq \epsilon.$$ The maximum of this quantity over any choice of $S$ will be denoted $\mathcal{N}_{z,\infty}(\epsilon,\mathcal{F},m)$. 

The following lemma relies on classic concatenation techniques~\cite{bartlett2017spectrally}.
\begin{lemma}
\label{concat}
	Suppose that $\mathcal{F}$ is such that $\|f(x)\|_z\leq \kappa$ for any $x\in \mathcal{X}$, then 
	\label{chain}
	\begin{align}
	\mathcal{N}_{\infty}(\epsilon,\mathcal{H}\circ \mathcal{F},m)\leq \mathcal{N}_{z,\infty}(\frac{\epsilon}{2\Lambda},\mathcal{F},m)\times \mathcal{N}_{\infty} (\epsilon/2,\mathcal{H},\kappa, m),
	\end{align}
	where $\mathcal{N}_{\infty} (\epsilon/2,\mathcal{H},\kappa, m)$ denotes the maximum value of $\mathcal{N}_{\infty} (\epsilon/2,\mathcal{H},\widetilde{S})$ over all $\widetilde{S}\subset Z$ with $	|\widetilde{S}|=m$ and $|s|_z\leq \kappa\quad \forall s\in \widetilde{S}$.
\end{lemma}

\begin{proof}
	Let $N_1=\mathcal{N}_{z,\infty}(\epsilon/2\Lambda,\mathcal{F},m)$ and $N_2= \mathcal{N}_{\infty} (\epsilon/2,\mathcal{H},m)$. 
	
	Let $S=\{x_1,\ldots,x_m\}\subset \mathcal{X}$ be any sample set. 
	Let  $\{f^1,\ldots,f^{N_1}\}$ be the corresponding $(z,\infty)$ cover.
	For each $j\leq N_1$, we can define the ``training set'' $S_j:=\{f^{j} (x_1),\ldots f^j(x_m)\}$ and the corresponding  $\ell_\infty$ cover $C_j=\{h^j_1,\ldots,h^j_{N_2}\}$ of $\mathcal{H}$ with granularity $\epsilon/2$.  Let $D_j=C_j\circ f^j:=\{h^j_1\circ f^j,\ldots,h^j_{N_2}\circ f^j\}$.
	We will show that the cover $\cup_{j=1}^{N_1} D_j=\{h^j_k\circ f^j: j\leq N_1,k\leq N_2\}$ is an $\epsilon$ cover of $\mathcal{H}\circ \mathcal{F}$ with respect to the $\ell_\infty$, which implies the stated result. 
	To see this, observe that for any $h\circ f\in \mathcal{H}\circ \mathcal{F},$ we can choose $j\leq N_1$ such that 
	$\|f^j(x_i)-f(x_i)\|_z\leq \epsilon/2\Lambda$ for all $i\leq n$. We can now also choose an element $h^j_k$ ($k\leq N_2$) from the cover $C_j$  such that for all $i\leq n, |h(f^{j}(x_i))-h^k(f^{j}(x_i))|\leq \epsilon/2 $. 
	We now have that for any $i\leq n$, 
	\begin{align*}
	|(h\circ f)(x_i)-(h^{k}\circ f^j)(x_i)|&\leq	|(h\circ f)(x_i)-(h\circ f^j)(x_i)|+|(h\circ f^j)(x_i)-(h^{k}\circ f^j)(x_i)|\\
	&= |\langle h, f(x_i)-f^j(x_i)\rangle | +|h((f^j)(x_i))-h^k(f^j)(x_i))|\\
	&\leq  \|h\|_* \frac{\epsilon }{2\Lambda} + \epsilon/2\leq  \Lambda \frac{\epsilon }{2\Lambda} + \epsilon/2= \epsilon,
	\end{align*}
	as expected. At the last line, we have used the duality between the norms $\|.\|_z$ and $\|.\|_*$.
	
\end{proof}
% \begin{comment}[with different definition of application of NN]
% fdfd

% Back in our structured output prediction setting, let us consider the situation where the features $\Psi_f(x,y)$ are obtained from a neural network or another parametric method: $\Psi_f(x,y_f)=\Psi_f^{W_f}(x,y_f)$, where $W_f$ denotes a parameter set chosen from a set $ \mathcal{W}_f\subset \mathbb{R}^{\bar{D}}$. We make the natural assumption that the feature space of $\Psi_f(x,y_f)$ is the direct sum of $d$ subspaces $S_{y_f}$ for $y_f\in \mathcal{Y}_f$, and that $\Psi_f(x,y_f)|_{S_{y'}}=0$ for all $y'\neq y_f$ and $\Psi_f(x,y_f)|_{S_{y_f}}$ 
% depends only on $x$, via the network with weight matrices $ W_f = \Vecc(A^1,\ldots,A^L)$ (where $\Vecc(.)$ denotes the vectorisation of a tensor), so that 
% $$\Psi_f^{W_f}(x,y_f)|_{S_{y_f}}=A^L\sigma(A^{L-1}\sigma(\ldots \sigma(A^1x))).$$

% \end{comment}

Back in our structured output prediction setting, let us consider the situation where the features $\Psi_f(x,y)$ are obtained from a neural network or another parametric method: the features $\Psi_f(x,y_f)$ can be read from the $(f,y_f,.)$ components of the three-way tensor $\Psi^W(x)$, where $W$ denotes a parameter set chosen from a set $ \mathcal{W}\subset \mathbb{R}^{\bar{D}}$. We will write $\mathcal{F}$ for the function class $\mathcal{F}=\{\Psi^W: W\in \mathcal{W}\}$. For instance, $W$ can be a vectorization of the weights of the neural network.

%We make the natural assumption that the feature space of $\Psi_f(x,y_f)$ is the direct sum of $d$ subspaces $S_{y_f}$ for $y_f\in \mathcal{Y}_f$, and that $\Psi_f(x,y_f)|_{S_{y'}}=0$ for all $y'\neq y_f$ and $\Psi_f(x,y_f)|_{S_{y_f}}$ 
%depends only on $x$, via the network with weight matrices $ W_f = \Vecc(A^1,\ldots,A^L)$ (where $\Vecc(.)$ denotes the vectorisation of a tensor), so that 
%$$\Psi_f^{W_f}(x,y_f)|_{S_{y_f}}=A^L\sigma(A^{L-1}\sigma(\ldots \sigma(A^1x))).$$

Define the augmented dataset  $$\bar{S} := \{ (x,f,y_f): x \in S_{|\X},  f \in F, y \in \Y_f \}.$$
Define the function classes 
\begin{align}
\bar{\mathcal{H}}:=  \{(x,f,y_f) \mapsto \left<w,\Psi_f(x,y_f)\right>: w \in \R^D,\|w\|_p\le \Lambda,  (x,f,y_f)\in \bar{S}\}
\end{align}
and 
\begin{equation}
  \label{eq:linear_class_new}
  \bar{\mathcal{F}}_{p,\Lambda,\rho} = \{ (x,y) \mapsto L_\rho(x,y,h) : \|w\|_p \le \Lambda , W\in \mathcal{W}, (x,y) \in \X  \times \Y\}.
\end{equation}
 Similarly to Theorem~\ref{thm:ours}, it is easy to show that: 
   \begin{equation}
    \label{eq:covering_bound_1_new}
        \log\N(\epsilon,\bar{F}_{p,\Lambda,\rho},S)) \le \log \N\Big(\frac{\rho}{2|F|}\epsilon,\bar{\mathcal{H}}_{p,\Lambda},\bar{S}\Big).
      \end{equation}

Thus, assuming we have a way of obtaining bounds for the covering number of the function class corresponding to $\mathcal{W}$, we can use Lemma~\ref{chain} in combination with the other techniques in this paper to obtain a generalisation bound valid for all choices of $W$ and all choices of $w\in \mathbb{R}^D$.

%{\color{red} What is $M_f$?}
For instance, suppose that $\mathcal{W}$ is a $\bar{D}$-dimensional ball of radius $1$ with respect to some norm $\| \|_w$, and the Lipschitz constant of the map from $\mathbb{R}^{\bar{D}}$ to $(\mathbb{R}^{D})^{\mathcal{X}\times F\times \mathcal{Y}_f}$ which maps $W_f$ to $\Psi_.^{W}(.,.)$ is $\bar{B}$-Lipschitz with respect to the norms $\|.\|_w$ and $\|.\|_{q,\infty}$ (we choose $\|.\|_{z}=\|.\|_q$). This means that for any $W_1,W_2\in \mathcal{W}$ and for any $x_1\in \mathcal{X}$, $f\in F$ and $y_1\in \mathcal{Y}_f $, $\|\Psi^{W_1}_f(x_1,y_1)-\Psi^{W_2}_f(x_2,y_2)\|_q\leq \bar{B}\|W_1-W_2\|_w$.  
When this property holds, we say that the corresponding function class is \textit{$(\bar{D},\bar{L})$-parametrised} w.r.t. the relevant norm $\|.\|_q$. We write $\Psi^* = \sup_{f\in F,y \in \Y_f,x\in \X,W\in \mathcal{W}} \|\Psi^W_f(x,y)\|_q$.

Our assumption on the Lipschitz constant $\bar{B}$ implies that for any $\epsilon$, an $\epsilon$-cover of $\mathcal{W}$ gives rise to an $\epsilon \bar{B}$-cover of $\mathcal{F}$ with respect to the $\|\|_{q,\infty}$ norm (in this case, the same cover works for any training set). Such a cover of the ball $\mathcal{W}$ can easily be obtained from classic results  such as Lemma A.8 in~\cite{LongSeg}:
\begin{lemma}
\label{ball}
Let $d$ be a positive integer, $\|.\|$ be a norm, $\rho$ be the metric induced by it, and $\kappa,\epsilon>0$. A ball of radius $\kappa$ in $\mathbb{R}^d$ w.r.t. $\rho$ can be covered by $(\frac{3\kappa}{\epsilon})^d$ balls of radius $\epsilon$.
\end{lemma}

The following is then immediate: 
\begin{lemma}
Let $\mathcal{F}$ be a function class with outputs in a space $Z$ endowed with the norm $\|.\|_z$. Suppose that $\mathcal{F}$ is $(\bar{D},\bar{B})$-parametrized with respect to the norm $\|\|_z$. Then, for any training set $S=\{x_1,\ldots,x_m\}$, we have the following bound on the covering number of $ \mathcal{F}$: $$\log \mathcal{N}_{z,\infty}(\epsilon,\mathcal{F},S)\leq \bar{D}\log(3\bar{B}/\epsilon).$$
\end{lemma}
Note that the Lipschitz constant $\bar{B}$ only shows up in log terms, which means that any reasonable control on $\bar{B}$ is enough to yield satisfying generalisation bounds, and the dominant term in the first term of equation~\eqref{general} will be be $\bar{D}$, except in pathological cases.

In~\cite{LongSeg}, the Lipschitz constant $\bar{B}$ of convolutional neural networks was bounded in terms of the norms of the weight matrices. Adapting their results (section 3.1) and combining with Lemma~\ref{ball} above, we obtain
\begin{lemma}
Consider a neural network architecture with $\bar{D}$ parameters and $D_1$ outputs where the output layer is equipped with the $L_\infty$ norm. We suppose that the $L^2$ norms of the inputs are bounded by $\chi$, and consider for each $\beta,\nu>0$ the class $\mathcal{F}$ of networks with $\ell$ layers whose weights satisfy the following conditions: (1) the spectral norms of each layers are bounded by $1+\nu$, and (2) the sum of the spectral norms of the differences between the weight matrices and their initialised values is less than $\beta$.
For any $\epsilon$ and any training set $S=\{x_1,\ldots,x_m\}$, we have 
\begin{align}
    \log \mathcal{N}_\infty (\epsilon, \mathcal{F},S)\leq \bar{D}\log(3N/\epsilon),
\end{align}
where $N=\chi \beta (1+\nu+\beta/\ell)^\ell$.
\end{lemma}

In our structured output prediction setting, we now precisely define $\mathcal{W}$ to be the set of weights satisfying the conditions (1) and (2) above. 
After noting that the $\|.\|_\infty$  and $\|.\|_q$ norms on the feature space $Z$ are within a factor of $D$ of each other and applying our Lemma~\ref{chain}, we  have the following bound on the covering number of $\bar{\mathcal{F}}_{p,\Lambda,\rho}$:
\begin{align}
\label{general}
\log(\mathcal{N}_\infty(\epsilon,\bar{\mathcal{F}}_p,\bar{S}))\leq \bar{D}\log[12N\Lambda D/\epsilon\rho]+ 576 \frac{\sqrt{q-1}(\Psi^*)^2 \Lambda^2 |F|^2}{\epsilon^2\rho^2} \log[2\lceil16\Psi^* \Lambda|F|/\epsilon\rho +2 \rceil md|F| + 1].
\end{align}

%Note that the Lipschitz constant $\bar{B}$ only shows up in log terms, which means that any reasonable control on $\bar{B}$ on it is enough to yield satisfying generalisation bounds, and the dominant term in the first term of equation~\eqref{general} will be be $\bar{D}|F|$, except in pathological cases.
%To be explain a particular case more specifically, note that has been shown in~\cite{LongSeg} that for convolutional neural networks with $\bar{D}$ parameters and a set $\mathcal{W}_f$ defined via suitable norms, $\bar{B}$ can be expressed as a function of norm-based quantities in the network. 

Plugging this back into~\ref{general} and applying Dudley's entropy theorem similarly to the proof of Theorems~\ref{thm:main} and~\ref{thm:gen-hp}, it is straightforward to obtain
\[    R(h ) \le R_S(h)  + \widetilde{O}(\frac{\sqrt{q-1}\Psi^* \Lambda |F|}{\rho\sqrt{m}})+\widetilde{O}(\sqrt{\frac{\bar{D}}{m}}\log^{\frac{1}{2}}(N\Lambda |F|D/\rho)) 
    + O(\sqrt{\frac{\log{\frac{1}{\delta}}}{m}}),\]
where the notation $\widetilde{O}$ hides logarithmic factors, and as above,  $N=\chi \beta (1+\nu+\beta/\ell)^\ell$ with $\chi=\sup_{x\in \mathcal{X}}\|x\|_{2}$, the spectral norms of each layers are bounded by $1+\nu$, $\ell$ is the number of layers, and the sum of the spectral norms of the differences between the weight matrices and their initialised values is less than $\beta$.

%Finally, we note that other results for neural networks can also be applied instead of~\cite{LongSeg}. For instance, $L_\infty$ covering numbers for the the function class associated with convolutional neural networks were provided in~\cite{AAAI_conv}, and can be combined similarly to obtain variations of the formula above where the parameter count$\bar{D}$ is replaced by a norm-based equivalent. 

\textbf{Remark:} Whilst the above bounds rely on the parameter-counting strategy from~\cite{LongSeg} to bound the complexity of the feature-extracting network, other approaches to that sub problem are perfectly compatible with our framework.

For instance, norm-based bounds on the feature-extracting network, relying on results from from~\cite{bartlett2017spectrally,early} etc. (for fully connected networks) or from~\cite{AAAI_conv} (for CNNs) can also be plugged into our proof, yielding results with the properties as above in terms of the (lack of) dependence on the number of factors but with various norms of the weights of feature-extracting network replacing the parameter-count term $\bar{D}$.

% \bibliographystyle{named}       

% \bibliography{references}       
% \end{document}

\end{document}